\documentclass{article}

\usepackage[margin=1in]{geometry}
\PassOptionsToPackage{numbers, compress}{natbib}

\def\shownotes{1}  \ifnum\shownotes=1
\newcommand{\authnote}[2]{{[#1: #2]}}
\else
\newcommand{\authnote}[2]{}
\fi

\usepackage[utf8]{inputenc} 
\usepackage[T1]{fontenc}    
\usepackage{hyperref}       
\usepackage{url}            
\usepackage{booktabs}       
\usepackage{amsfonts}       
\usepackage{nicefrac}       
\usepackage{microtype}      

\usepackage{algorithm}
\usepackage[noend]{algpseudocode}
\usepackage{float}
\newfloat{algorithm}{t}{lop}

\usepackage{booktabs}       
\usepackage{amsfonts}       
\usepackage{microtype}      
\usepackage{xcolor}
\usepackage{url}
\usepackage{verbatim} 
\usepackage{graphicx}
\usepackage{natbib}
\usepackage{multirow}
\usepackage{xspace}
\usepackage{epsfig}
\usepackage{amsmath}
\usepackage{amsthm}
\usepackage{amssymb}
\usepackage{times}
\usepackage{xr}
\usepackage{bbm}
\usepackage{bm}
\usepackage{subcaption}
\usepackage{enumitem}
\usepackage{hyperref}       
\usepackage{multicol}
\usepackage{tabularx}
\usepackage{perpage} 
\usepackage{wrapfig}
\usepackage{bbm}
\usepackage{commath}

\newcommand{\cut}[1]{}

\newtheorem{definition}{Definition}

\newtheorem{lemma}{Lemma}

\makeatletter
\newtheorem*{rep@theorem}{\rep@title}
\newcommand{\newreptheorem}[2]{%
\newenvironment{rep#1}[1]{%
 \def\rep@title{#2 \ref{##1}}%
 \begin{rep@theorem}}%
 {\end{rep@theorem}}}
\makeatother
\newtheorem{theorem}{Theorem}
\newreptheorem{theorem}{Theorem}

\DeclareMathOperator{\argmin}{argmin}
\DeclareMathOperator{\cspan}{span}
\DeclareMathOperator{\Beta}{Beta}
\DeclareMathOperator{\Reg}{Reg}

\DeclareMathOperator{\Bernoulli}{Bernoulli}
\DeclareMathOperator{\Bin}{Bin}
\DeclareMathOperator{\supp}{supp}
\DeclareMathOperator{\clip}{clip}
\DeclareMathOperator{\spanl}{span}
\newcommand{\EE}{\mathbb{E}}
\newcommand{\RR}{\mathbb{R}}

\newcommand{\PP}{\mathbb{P}}
\newcommand{\inner}[2]{\langle #1, #2 \rangle}
\def\polylog{\operatorname{polylog}}

\newcommand{\Acal}{\mathcal{A}}
\newcommand{\Bcal}{\mathcal{B}}

\newcommand{\Fcal}{\mathcal{F}}
\newcommand{\Hcal}{\mathcal{H}}
\newcommand{\Ncal}{\mathcal{N}}
\newcommand{\Qcal}{\mathcal{Q}}
\newcommand{\Xcal}{\mathcal{X}}
\newcommand{\Ycal}{\mathcal{Y}}

\newcommand{\one}{\mathbbm{1}}
\newcommand{\defeq}{:=}

\newcommand{\unif}{\mathrm{uniform}}
\newcommand{\qufur}{\ensuremath{\text{QuFUR}}\xspace}
\newcommand{\fbqufur}{\ensuremath{\text{Fixed-Budget QuFUR}}\xspace}
\newcommand{\order}[1]{O\del{#1}}
\newcommand{\otil}[1]{\tilde{O}\del{#1}}
\newcommand{\E}{\mathrm{E}}
\newcommand{\R}{\mathrm{R}}
\newcommand{\A}{\mathrm{A}}
\newcommand{\Ecal}{\mathcal{E}}
\newcommand{\wbar}[1]{\overline{#1}}

\title{Active Online Learning with Hidden Shifting Domains}

\author{
  Yining Chen\\Stanford University\\\texttt{cynnjjs@stanford.edu}
  \and Haipeng Luo\\University of Southern California\\\texttt{haipengl@usc.edu}
  \and Tengyu Ma\\Stanford University\\\texttt{tengyuma@stanford.edu}
  \and Chicheng Zhang\\University of Arizona\\\texttt{chichengz@cs.arizona.edu}
}
\date{}
\begin{document}

\maketitle

\begin{abstract}
Online machine learning systems need to adapt to domain shifts. Meanwhile, acquiring label at every timestep is expensive. Motivated by these two challenges, we propose a surprisingly simple algorithm that adaptively balances its regret and its number of label queries in settings where the data streams are from a mixture of hidden domains. For online linear regression with oblivious adversaries, we provide a \textit{tight} tradeoff that depends on the durations and dimensionalities of the hidden domains. Our algorithm can adaptively deal with interleaving spans of inputs from different domains. We also generalize our results to non-linear regression for hypothesis classes with bounded eluder dimension and adaptive adversaries. Experiments on synthetic and realistic datasets demonstrate that our algorithm achieves lower regret than uniform queries and greedy queries with equal labeling budget.
\end{abstract}

\section{Introduction}\label{sec:intro}

In statistical learning, model performance often significantly drops when the testing distribution drifts away from the training distribution~\citep{torralba2011unbiased,Recht, recht2019imagenet, engstrom2020identifying}. Online learning addresses worst-case domain shift by assuming the data is given by an adversary~\citep{Hazan16}. However, practical deployments of fully-online learning systems have been somewhat limited, because labels are expensive to obtain; see~\citep{Strickland2018AIhumanPT} for an example in fake news detection. A label budget linear in time is too much of a luxury.

\citet{cesa2004minimizing} study label-efficient online learning for prediction with expert advice. Their algorithm queries the label of every example with a fixed probability, which, as they show, achieves minimax-optimal regret and query complexity for this problem. However, querying with uniform probability does not take into account the algorithm's uncertainty on each individual example, and thus can be suboptimal when the problem has certain favorable structures. For example, a sequence of online news may come from the mixture of a few topics or trends, and some news topics may require more samples to categorize well compared to others.

We aim to improve label-efficiency in online learning by exploiting hidden domain structures in the data. We assume that each input is from one of $m$ unknown and potentially overlapping domains (e.g. news topics). For each input, the learner makes a prediction, incurs a loss, and decides whether to query its label. The regret of the learner is defined as the difference between its cumulative loss and that of the best fixed predictor in hindsight. Our goal is to trade off between regret and query complexity: given a fixed label budget, we hope to incur the smallest regret possible.

In statistical learning, domains usually refer to data distributions. One commonly-studied type of domain shift is {\em covariate shift} ~\citep[see e.g.][]{sugiyama2007covariate}, where the conditional distribution $P(y|x)$ is fixed, but the marginal distribution $P(x)$ changes across domains. In online learning, however, the inputs are usually not independently drawn from a distribution, and can even be adversarially chosen from a certain support. We thus propose to study {\em support shift} in online learning as the natural counterpart of covariate shift in statistical learning. Specifically, we assume that inputs from each domain $u \in [m]$ have an unknown support $\Xcal_u$ and an unknown total duration $T_u$, and inputs from different domains are interleaved as a stream fed to the online learner.

Support shift is common for high-dimensional vision / language datasets where domain supports have little overlap. As a motivating example, consider an online regression task of sentiment prediction for twitter feeds. Here, a domain is a (not necessarily contiguous) subsequence of tweets around a certain topic. Different topics contain disjoint keywords. As the hot topics change over time, we likely receive inputs from different domains for varying time periods.

Similar to covariate shift, we assume realizability, i.e., there exists a predictor that is Bayes optimal across all the domains.
Realizability is a reasonable assumption in modern machine learning for two reasons. 
First, high-dimensional features are usually of high quality; for example, adding only one additional output layer to pre-trained embeddings such as BERT obtain state-of-the-art results on many language tasks~\citep{devlin2018bert}.
Second, models are often overparameterized~\citep{zhang2016understanding}. Thus, the model can rely on different features in different domains. In the sentiment regression task, the model can combine positive or negative words in all domains to predict well. In statistical learning, domain adaptation methods assuming that a single model can perform well on different domains~\citep{ganin2016domain} indeed have been empirically successful on high-dimensional datasets.

Under this setup, we propose \qufur (Query in the Face of Uncertainty for Regression), a surprisingly simple query scheme based on uncertainty quantification. We start with online linear regression from $\RR^d$ to $\RR$ with an oblivious adversary. With additional regularity conditions, we provide the following regret guarantee of \qufur with label budget $B$: for {\em any} $m$ and {\em any} partition of $[T]$ into domains $I_1, \ldots, I_m$, 
with $T_u$ being the number of examples from domain $I_u$ and $d_u$ being the dimension of the space spanned by these examples,
the regret of is \qufur is $\tilde{O}((\sum_{u=1}^m{\sqrt{d_u T_u}})^2/B)$ (Theorem~\ref{thm:fixBudget}).\footnote{Throughout this paper, $[n]$ denotes the set $\cbr{1,\ldots,n}$; notations $\tilde{O}$ and $\tilde{\Omega}$ hide logarithmic factors.}

When choosing $m = 1$ and $I_1 = [T]$, we see that the regret of QuFUR is at most $\tilde{O}(dT/B)$, matching minimax lower bounds (Theorem~\ref{thm:lower-unstructured}) in this setting.
The advantage of \qufur's adaptive regret guarantees becomes significant when the domains have heterogeneous time spans and dimensions: $(\sum_{u=1}^m{\sqrt{d_u T_u}})^2$ can be substantially less than $dT$ when the $T_u / d_u$'s are heterogeneous across different $u$'s.
For example, if $m=2$, $d_1 = d$, $d_2 = 1$, $T_1 = d$, and $T_2 = T-d$, 
then the resulting regret bound is of order $O(T + d^2)$ which can be much smaller than $dT$ when $1 \ll d \ll T$.
Using standard online-to-batch conversion~\citep{cesa2004generalization}, we also obtain novel results in batch active learning for regression (Theorem~\ref{cor:otb-linear}). Furthermore, we also define a stronger notion of minimax optimality, namely {\em hidden domain minimax optimality}, and show that QuFUR is optimal in this sense (Theorem~\ref{thm:lower}), for a wide range of domain structure specifications.

We generalize our results to online regression with general hypothesis classes against an adaptive adversary. We obtain a similar regret-query complexity tradeoff, where the analogue of $d_u$ is (roughly) the \textit{eluder dimension}~\citep{VR} of the hypotheses class with respect to the support of domain $u$ (Theorem~\ref{thm:nonlin}).

Experimentally, we show that our algorithm outperforms the baselines of uniform and greedy query strategies, on a synthetic dataset and three high-dimensional language and image datasets with realistic support shifts. Our code is available online at \url{https://github.com/cynnjjs/online_active_AISTATS}.

\section{Related works}
\paragraph{Active learning.} We refer the readers to~\citet{balcan2009agnostic, hanneke2014theory, dasgupta2008general, beygelzimer2010agnostic} and the references therein for background on active learning. For classification, a line of works~\citep{dasgupta2008hierarchical,minsker2012plug,kpotufe2015hierarchical,locatelli2017adaptivity} performs hierarchical sampling for nonparametric active learning. The main idea is to maintain a hierarchical partitioning over the instance domain (either a pre-defined dyadic partition or a pre-clustering over the data), and performs adaptive label querying with partition-dependent probabilities. For regression, many works~\citep{fedorov2012model, chaudhuri2015convergence} study the utility of active learning for maximum likelihood estimation in the realizable setting. Recent works also study active linear regression in nonrealizable~\citep{drineas2006sampling,derezinski2018leveraged,derezinski2018reverse,sabato2014active} and heteroscedastic~\citep{chaudhuri2017active,fontaine2019active} settings.
These works do not consider domain structures except for~\citet{sabato2014active}, who propose a domain-aware stratified sampling scheme. Their algorithm needs to know the domain partition a priori, whose quality is crucial to ensure good performance.

\paragraph{Active learning for domain adaptation.} 
The empirical works of \citet{rai2010domain,saha2011active, xiao2013online} study stream-based active learning when inputs comes from pre-specified source and target distributions.~\citet{su2020active} combine domain adversarial neural network (DANN) with active learning, where the discriminator in DANN serves as a density ratio estimator that guides active sampling. In contrast, our algorithm handles multiple domains, does not assume iid-ness for inputs from a domain, and does not require knowledge of which domain the inputs come from.

\paragraph{Active online learning.} Earlier works on selective sampling when iid data arrive in a stream and a label querying decision has to be made after seeing each example~\citep{cohn1994improving,dasgupta2008general,hanneke2011rates} implicitly provide online regret and label complexity guarantees.
Works on worst-cast analysis of selective sampling for linear  classification~\citep{cesa2006worst} provide regret guarantees similar to that of popular online linear classification algorithms such as Perceptron and Winnow, but their label complexity guarantees are runtime-dependent and therefore cannot be easily converted to a guarantee that only involves problem parameters defined apriori.
Subsequent works~\citep{Bianchi,dekel2010robust,cavallanti2011learning,agarwal2013selective} study the setting where there is a parametric model on $P(y|x, \theta)$ with unknown parameter $\theta$, and the $x$'s shown can be adversarial. Under those assumptions, they obtain regret and query complexity guarantees dependent on the fraction of examples with low margins. 
\citet{Liu} gives a worst-case analysis of active online learning for classification with drifting distributions, under the assumption that the Bayes optimal classifier is in the learner's hypothesis class. In contrast, our work gives adaptive regret guarantees in terms of the hidden domain structure in the data, and focuses on regression instead of classification.

\paragraph{KWIK model.} In the KWIK model~\citep{KWIK}, at each time step, the algorithm is asked to either abstain from prediction and query the label, or predict an output with at most $\epsilon$ error. In contrast, in our setting, the learner's goal is to minimize its cumulative regret, as opposed to making pointwise-accurate predictions. \citet{Bianchi} study linear regression in the KWIK model, and propose the BBQ sampling rule; our work can be seen as analyzing a variant of BBQ and showing its adaptivity to domain structures.
\citet{szita2011agnostic} propose an algorithm that works in an agnostic setting, where the error guarantee at every round depends on the agnosticity of the problem. A relaxed KWIK model that allows a prespecified number of mistakes has been studied in~\cite{sayedi2010trading,zhang2016extended}.

\paragraph{Adaptive/Switching Regret.} Adaptive regret~\citep{Adaptive, daniely2015strongly} is the excessive loss of an online algorithm compared to the locally optimal solution over any continuous timespan. Our algorithm can be interpreted as being competitive with the locally optimal solution on every domain, even if the timespans of the domains are not continuous, which is closer to the concept of switching regret with long-term memory studied in e.g,~\citep{bousquet2002tracking, zheng2019equipping}. Switching regret bounds typically have a polynomial dependence on the number of domain switches, which does not appear in our bounds. However, the above works allow target concept to shift over time, whereas our bounds require realizability and thus compete with a fixed optimal concept.
Overall, we achieve a stronger form of guarantee under more assumptions.

\paragraph{Online linear regression.} Literature on fully-supervised online linear regression has a long history~\citep{vovk2001competitive,azoury2001relative}. As is implicit in~\citet{cesa2004minimizing}, we can reduce from fully-supervised online regression to active online regression by querying uniformly randomly with a fixed probability. Combining this reduction with existing online linear regression algorithms~\citep{Newton}, we get $\tilde{O}(dT/B)$ regret with $O(B)$ queries for any $B \leq T$. Our bound matches this in the realizable and oblivious setting when there is one domain, and is potentially much better with more domain structures.
\section{Setup and Preliminaries}

\subsection{Setup}

\paragraph{Active online regression with domain structure.} Let $\Fcal=\{f: \Xcal \rightarrow [-1, 1]\}$ be a hypothesis class. We consider a realizable setting where $y_t =f^*(x_t)+ \xi_t$ for some $f^* \in \Fcal$ and random noise $\xi_t$. The adversary decides $f^* \in \Fcal$ before the interaction starts, and $\xi_t$'s are independent zero-mean sub-Gaussian random variables with variance proxy $\eta^2$.

The example sequence $\cbr{x_t}_{t=1}^T$ has the following domain structure unknown to the learner: $[T]$ can be partitioned into $m$ disjoint nonempty subsets $\cbr{I_u}_{u=1}^m$, where for each $u$, $|I_u|=T_u$, and  $\cbr{x_t}_{t \in I_u}$ lie in a subspace of dimension $d_u$.

The interaction between the learner and the adversary follows the protocol below.

For each $t = 1, \dots, T$:
\begin{enumerate}
    \item Example $x_t$ is revealed to the learner.
    \item The learner predicts $\hat{y}_t = \hat{f}_t(x_t)$ using predictor $\hat{f}_t: \Xcal \to [-1, 1]$, incurring loss $(\hat{y}_t - y_t)^2$.
    \item The learner sets a query indicator $q_t \in \{0, 1\}$. If $q_t=1$, $y_t$ is revealed.
\end{enumerate}

The performance of the learner is measured by its number of queries $Q=\sum_{t=1}^T{q_t}$, and its regret $R=\sum_{t=1}^T{(\hat{y}_t-f^*(x_t))^2}$. By our realizability assumption, our notion of regret coincides with the one usually used in online learning when  expectations are taken; see Appendix~\ref{sec:reg-def}.
Our goal is to design a learner that has low regret $R$ subject to a budget constraint: $Q \leq B$, for some fixed budget $B$.

\paragraph{Oblivious vs. adaptive adversary.} In the \textit{oblivious} setting, the adversary decides the sequence $\{x_t\}_{t=1}^T$ before the interaction starts. In the \textit{adaptive} setting, the adversary can choose $x_t$ depending on the history $H_{t-1} =\{x_{1:t-1}, \hat{f}_{1:t-1}, \xi_{1:t-1}\}$.

\paragraph{Miscellaneous notations.} For a vector $v \in \RR^d$ and a positive semi-definite matrix $M \in \RR^{d \times d}$, define $\| v \|_M \defeq \sqrt{ v^\top M v }$. For vectors $\cbr{z_t}_{t=1}^T \subseteq \RR^l$, and $S = \cbr{i_1, \ldots, i_n} \subseteq [T]$, denote by $Z_S$ the $n \times l$ matrix whose rows are $z_{i_1}^\top, \ldots, z_{i_n}^\top$. Define $\clip(z) \defeq \min(1,\max(-1, z))$ and $\tilde{\eta} \defeq \max\{1, \eta\}$. For a set of vectors $S$, define $\cspan(S)$ as the linear subspace spanned by $S$.

\subsection{Baselines for linear regression}
\label{sec:baselines}

We first study linear regression with an oblivious adversary, and then generalize to the non-linear case with an adaptive adversary in Section~\ref{sec:nonlin}. For now, hypothesis class $\Fcal$ is $\{x \mapsto \inner{x}{ \theta}:\theta \in \RR^d, \|\theta\|_2 \le C\}$. Let the ground truth hypothesis be $f^*(x) = \inner{\theta^*}{x}$, where $\theta^* \in \RR^d$, and input space $\Xcal$ be a subset of $\{x \in \RR^d: \|x\|_2 \le 1, \langle x, \theta^* \rangle \le 1\}$.\footnote{The constraint $\| x\|_2 \le 1$ can be relaxed by only increasing the logarithmic terms in the regret and query complexity guarantees. }

\paragraph{Uniform querying is minimax-optimal with no domain structure.} As a starter, consider an algorithm that queries every label and predicts using the regularized least squares estimator $\hat{\theta}_t=\argmin_{\theta}{\sum_{i=1}^{t-1}{(\langle \theta, x_i\rangle-y_i)^2}+\lambda \|\theta\|^2}$, where $\lambda=1/C^2$. It is well-known from~\citep{vovk2001competitive, azoury2001relative} that (a variant of) this fully-supervised algorithm achieves $R =\tilde{O}(\tilde{\eta}^2 d)$ with $Q=T$.
Consider an active learning extension of the above algorithm that queries labels independently with probability $B/T$, and predicts with the regularized least squared estimator computed based on all queried examples $\hat{\theta}_t=\argmin_{\theta}{\sum_{i \in [t-1], q_i=1}{(\langle \theta, x_i\rangle-y_i)^2}+\lambda \|\theta\|^2}$. We show that the above active online regression strategy achieves $\EE[R] =\tilde{O}(\tilde{\eta}^2 d T/B)$ with $\EE[Q] = B$ in Appendix~\ref{sec:proof_uniform}. 
As shown in Theorem~\ref{thm:lower-unstructured}, this tradeoff is minimax optimal if $\tilde{\eta}$ is a constant.
Although this guarantee is optimal in the worst case, one major weakness is that it is too pessimistic: as we will see next, when the data has certain hidden domain structure, the learner can achieve substantially better regret guarantees than the worst case if given access to auxiliary domain information.

\paragraph{Oracle baseline when domain structure is known.} Suppose the learner is given the following piece of knowledge from an oracle: there are $m$ domains; for each $u$ in $[m]$, there are a total of $T_u$ examples from domain $u$ from a subspace of $\RR^d$ of dimension $d_u$. In addition, for every $t$, the learner is given the index of the domain to which example $x_t$ belongs. In this setting, the learner can combine the aforementioned regularized least squares linear predictor with the following domain-aware querying scheme: for any example in domain $u$, the learner queries its label independently with probability $\mu_u \in (0,1]$. Within domain $u$, the learner incurs $O(\mu_u T_u)$ queries and  $\tilde{O}(\tilde{\eta}^2 d_u/\mu_u)$ regret. Summing over all $m$ domains, its achieves a label complexity of $O(\sum_{u=1}^m \mu_u T_u)$ and a regret bound of $\tilde{O}(\tilde{\eta}^2 \sum_{u=1}^m {d_u/\mu_u})$. This motivates the following optimization problem:
\begin{align*}
\label{opt_program}
\min_\mu & \sum_{u=1}^m  {d_u/\mu_u},  
\text{ s.t. } \sum_{u=1}^m \mu_u T_u \leq B, \mu_u \in [0,1], \forall u \in [m].
\end{align*}
i.e., we choose domain-dependent query probabilities that minimize the learner's total regret guarantee, subject to its query complexity being controlled by $B$.
When $B \leq (\sum_{u=1}^m \sqrt{d_u T_u})\min_u{\sqrt{T_u/d_u}}$, the optimal $\mu_u$ is $\sqrt{d_u/T_u} \cdot \frac{B}{\sum_{u=1}^m \sqrt{d_u T_u}}$, i.e.
$\mu_u$ is proportional to $\sqrt{d_u/T_u}$.\footnote{For larger budget $B > \sum_{u=1}^m \sqrt{d_u T_u} \cdot \min_u{\sqrt{T_u/d_u}}$, there exists a threshold $\tau$ such that the optimal solution is $\mu_u=1$ for  $\{u \in [m]:\sqrt{d_u/T_u} \ge \tau\}$, and $\mu_u \propto \sqrt{d_u/T_u}$ for  $\{u \in [m]:\sqrt{d_u/T_u} < \tau\}$; see Appendix~\ref{sec:large_budget}.
}
This yields a regret guarantee of  $O(\tilde{\eta}^2 (\sum_u \sqrt{d_u T_u})^2 /B)$. 

Although this strategy can sometimes achieve much smaller regret than uniform querying (as we have discussed in Section~\ref{sec:intro}, $(\sum_u \sqrt{d_u T_u})^2$ could be substantially smaller than $dT$), it has two clear drawbacks:
first, it is not clear if this guarantee is always no worse than uniform querying, especially when $\sum_{u=1}^m d_u \gg d$; second, the domain memberships of examples are rarely known in practice. In the next section, we develop algorithms matching the performance of this domain-aware query scheme without these drawbacks.

\section{Active online learning for linear regression: algorithms, analysis, and matching lower bounds}
\label{sec:linear}
We start by presenting a parameterized algorithm in Section~\ref{sec:upper}, where the parameter $\alpha$ has a natural cost interpretation.
We then present a fixed-budget variant of it in Section~\ref{sec:fixedbudget}. Section~\ref{sec:lower} shows that our algorithm is minimax-optimal under a wide range of domain structure specifications.

\subsection{Main Algorithm: Query in the Face of Uncertainty for Regression (QuFUR)}
\label{sec:upper}
\begin{algorithm}
\caption{Query in the Face of Uncertainty for Regression (QuFUR($\alpha$)) \label{alg:1}}
\begin{algorithmic}[1]

\Require Total dimension $d$, time horizon $T$, $\theta^*$'s norm bound $C$, noise level $\eta$, parameter $\alpha$. 
\State $M \leftarrow \frac{1}{C^2}I$, queried dataset $\Qcal \leftarrow \emptyset$.
\For {$t=1$ to $T$}
    \State Compute regularized least squares solution $\hat{\theta}_t \leftarrow M^{-1} X_{\Qcal}^\top Y_{\Qcal}.$ 
    \State Let $\hat{f}_t(x) = \clip(\inner{\hat{\theta}_t}{x})$ be the predictor at time $t$, and 
    predict $\hat{y}_t \gets \hat{f}_t(x_t)$.
    \State Uncertainty estimate $\Delta_t \leftarrow  \tilde{\eta}^2 \min\{1, \|x_t\|_{M^{-1}}^2\}$.
    \State With probability $\min{\{1, \alpha \Delta_t\}}$, set $q_t \gets 1$; otherwise set $q_t \gets 0$.
    \If {$q_t=1$}
        \State Query $y_t$. $M \leftarrow M + x_t x_t^\top$, $\Qcal \leftarrow \Qcal \bigcup \{t\}$.
    \EndIf
\EndFor

\end{algorithmic}
\end{algorithm}

We propose \qufur (Query in the Face of Uncertainty for Regression), shown in Algorithm~\ref{alg:1}. At each time step $t$, the algorithm first computes $\hat{\theta}_t$, a regularized least squares estimator on the labeled data obtained so far, then predict using $\hat{f}_t(x) = \clip(\inner{\hat{\theta}_t}{x})$.
It makes a label query with probability proportional to $\Delta_t$, a high-probability upper bound of the instantaneous regret $(\hat{y}_t - \inner{\theta^*}{x_t})^2$ (see Lemma~\ref{lem:subseq} for details), which can also be interpreted as an uncertainty measure of $x_t$. Intuitively, when the algorithm is already confident about the current prediction, it saves its labeling budget for learning from less certain inputs in the future.
More formally, $\Delta_t \defeq \tilde{\eta}^2 \min(1,\|x_t\|_{M_t^{-1}}^2)$, where $M_t = \lambda I + \sum_{i \in \Qcal_t}{x_i x_i^\top}$, and $\Qcal_t$ is the set of labeled examples seen up to time step $t-1$. 
\qufur queries the label $y_t$ with probability $\min\cbr{1, \alpha \Delta_t}$, where $\alpha$ is a parameter that trades off between query complexity and regret.

Perhaps surprisingly, the simple query strategy of \qufur can leverage hidden domain structure, as shown by the following theorem.

\begin{theorem}
\label{thm:upper}
Suppose the example sequence $\cbr{x_t}_{t=1}^T$ has the following structure: $[T]$ can be partitioned into $m$ disjoint nonempty subsets $\cbr{I_u}_{u=1}^m$, where for each $u$, $|I_u|=T_u$, and  $\cbr{x_t}_{t \in I_u}$ lie in a subspace of dimension $d_u$. If Algorithm~\ref{alg:1} receives as inputs dimension $d$, time horizon $T$, norm bound $C$, noise level $\eta$, and parameter $\alpha$, then, with probability $1-\delta$:\\
1. Its query complexity is
\begin{align*}
    Q= \tilde{O}\left(  \sum_{u=1}^m \min\left\{T_u, \tilde{\eta} \sqrt{\alpha d_u T_u}\right\} + 1 \right).
\end{align*}
2. Its regret is $R= \tilde{O}\del{ \sum_{u=1}^m \max\{\tilde{\eta}^2 d_u,\tilde{\eta} \sqrt{d_u T_u/\alpha}\}}$.
\end{theorem}

The proof of the theorem is deferred to Section~\ref{sec:proof_upper}. For better intuition, we focus on the regime of $\alpha \in \intcc{ \frac{1}{\tilde{\eta}^2} \del{\frac{1}{\sum_u \sqrt{d_u T_u}}}^2, \frac{1}{\tilde{\eta}^2} \min_{u \in [m]} \frac{T_u}{d_u}}$, where the bounds become $Q= \tilde{O} ( \tilde{\eta}  \cdot \sqrt{\alpha} \sum_u \sqrt{d_u T_u}) $, and $R= \tilde{O}( \tilde{\eta} \cdot \sum_u {\sqrt{ d_u T_u}}/\sqrt{\alpha}) $.
We make a series of remarks below for this range of $\alpha$:

\paragraph{Novel notion of adaptive regret.} The above tradeoff is novel in that it holds for {\em any} meaningful domain partitions. Our proof actually shows that for any (not necessarily contiguous) subsequence $I \subseteq [T]$, QuFUR ensures $ Q= \tilde{O} ( \tilde{\eta} \cdot \sqrt{d_I |I|} \cdot \sqrt{\alpha}  ) $
and  
$ R= \tilde{O}(\tilde{\eta} {\sqrt{ d_I |I|}}) / \sqrt{\alpha} $ within $I$, where $d_I$ is the dimension of $\spanl(\cbr{x_t: t \in I})$. This type of guarantee is different from the adaptive regret guarantees provided by e.g.~\citet{Adaptive}, where the regret guarantee is only with respect to continuous intervals. However, note that the results in~\cite{Adaptive} do not require realizability.

\paragraph{Matching uniform querying baseline and minimax optimality.} Our tradeoff is never worse than the uniform querying baseline; this can be seen by applying the theorem with the trivial partition $\{[T]\}$, yielding $Q=\tilde{O}(\tilde{\eta} \sqrt{\alpha dT})$ and $R=\tilde{O}(\tilde{\eta} \sqrt{dT/\alpha})$. Therefore, same as the uniform query baseline, this guarantee is also minimax optimal for constant $\eta$, in light of Theorem~\ref{thm:lower-unstructured} in Appendix~\ref{sec:lower-unstructured}.

\paragraph{Matching oracle baseline and domain structure-aware minimax optimality.} \qufur matches the domain-aware oracle baseline discussed in Section~\ref{sec:baselines} even without prior knowledge of domain structure. Furthermore, we show in Theorem~\ref{thm:lower} below that for a wide range of problem specifications, this baseline, as well as \qufur, is minimax-optimal in our problem formulation with domain structure.

\paragraph{Fixed-cost-ratio interpretation.} The tradeoff in Theorem~\ref{thm:upper} can be interpreted in a fixed-cost-ratio formulation. Suppose a practitioner decides that the cost ratio between 1 unit of regret and 1 label query is $c:1$. The performance of the algorithm is then measured by its total cost $c R + Q$. Theorem~\ref{thm:upper} shows that QuFUR($\alpha$) balances the cost incurred by prediction and the cost incurred by label queries, in that
$Q \approx \alpha R$. 
We show in Appendix~\ref{sec:cost-model} that QuFUR with input $\alpha = c$ achieves near-optimal total cost, for a wide range of domain structure parameters.

\paragraph{Dependence on $\eta$.} Our query complexity and regret bounds have a 
dependence on $\tilde\eta = \max(\eta, 1)$. Similar dependence also appears in online least-squares regression literature~\citep{vovk2001competitive,azoury2001relative}.

\paragraph{Running time and extension to kernel regression.} The most computationally intensive operation in QuFUR is calculating $\Delta_t$ for each time, leading to a total time complexity of $O(T d^2)$ (since the update of $M^{-1}$ can be done in $O(d^2)$ via the Sherman-Morrison formula). For high dimensional problems with $d \gg n$, we can kernelise Algorithm~\ref{alg:1} following an approach similar to \citet{valko2013finite}, which has a time complexity of $O(T Q^2 k)$, assuming that evaluating the kernel function takes $O(k)$ time. See Appendix~\ref{sec:kernelisation} for more details.

\subsection{\qufur with a fixed label budget}
\label{sec:fixedbudget}
\begin{algorithm}
\caption{Fixed-Budget QuFUR \label{alg:2}}
\begin{algorithmic}[1]
\Require Total dimension $d$, time horizon $T$, label budget $B$, $\theta^*$'s norm bound $C$, noise level $\eta$.
\State $k \leftarrow \lceil 3\log_2{T} \rceil$.
\For {$i=0$ to $k$}
    \State Parameter $\alpha_i \leftarrow 2^i /T^2$.
\EndFor
\State Initialize $M \leftarrow \frac{1}{C^2}I$, $\Qcal \leftarrow \emptyset$.
\For {$t=1$ to $T$}
    \State Compute regularized least squares solution $\hat{\theta}_t \leftarrow M^{-1} X_{\Qcal}^\top Y_{\Qcal}$. 
    \State Let $\hat{f}_t(x) = \clip(\inner{\hat{\theta}_t}{x})$ be the predictor at time $t$, and 
    predict $\hat{y}_t \gets \hat{f}_t(x_t)$.
    \State Uncertainty estimate $\Delta_t \leftarrow \tilde{\eta}^2 \min\{1, \|x_t\|_{{M}^{-1}}^2\}$.
    \label{line:predict}
    \For {$i=0$ to $k$}
        \State Set $q_t^i=0$.
        \If {$\sum_{j=1}^{t-1}{q_j^i} < \lfloor B/k \rfloor$}
            \State With probability $\min{\{1, \alpha_i \Delta_t\}}$, set $q_t^i=1$.
        \EndIf
    \EndFor
    \If {$\sum_{i=0}^k {q_t^i}>0$}
        \State Query $y_t$. $M \leftarrow M + x_t x_t^\top$, $\Qcal \leftarrow \Qcal \bigcup \{t\}$.
    \EndIf
\EndFor

\end{algorithmic}
\end{algorithm}

The label complexity bound in Theorem~\ref{thm:upper} involve parameters $\{(d_u, T_u)\}_{u=1}^m$, which may be unknown in advance. In many practical settings, the learner is given a label budget $B$. For such settings, we propose a fixed-budget version of \qufur, Algorithm~\ref{alg:2}, that takes $B$ as input, and achieves near-optimal regret bound subject to the budget constraint, under a wide range of domain structure specifications.

Algorithm~\ref{alg:2} can be viewed as a master algorithm that aggregates over $k = O(\log{T})$ copies of $\qufur(\alpha)$. Each copy uses a different value of $\alpha$ from an exponentially increasing grid $\{2^i / T^2: i =0,\ldots,k\}$. The grid ensures that each copy still has label budget $\lfloor B/k \rfloor = \tilde{\Omega}(B)$, and there is always a copy that takes full advantage of its budget to achieve low regret. 
The algorithm queries whenever one of the copies issues a query, and predicts using a model learned on all historical labeled data. A copy can no longer query when its budget is exhausted. In the realizable setting, the regret of the master algorithm is no worse that of the copy running on a parameter $\alpha_i$ that make $\tilde{\Theta}(B)$ queries when run on its own; this insight yields the following theorem.

\begin{theorem}
\label{thm:fixBudget}
Suppose the example sequence $\cbr{x_t}_{t=1}^T$ has the following structure: $[T]$ can be partitioned into $m$ disjoint nonempty subsets $\cbr{I_u}_{u=1}^m$, where for each $u$, $|I_u|=T_u$, and  $\cbr{x_t}_{t \in I_u}$ lie in a subspace of dimension $d_u$. Also suppose $B \le \tilde{O}\left(\sum_u \sqrt{d_u T_u} \min_{u \in [m]} \sqrt{T_u/d_u} \right)$.
If Algorithm~\ref{alg:2} receives as inputs dimension $d$, time horizon $T$, label budget $B$, norm bound $C$, and noise level $\eta$, then:\\
1. Its query complexity $Q$ is at most $B$. \\
2. With probability $1-\delta$, its regret is \[R=\tilde{O}\left(\tilde{\eta}^2 \big(\sum_u{\sqrt{d_u T_u}}\big)^2/B\right).\]
\end{theorem}
The proof of the theorem is deferred to Appendix~\ref{sec:proof_fixBudget}. 
We now compare the guarantee of \qufur with the oracle baseline in Section~\ref{sec:baselines}: for any budget $B \le \tilde{O}(\sum_u \sqrt{d_u T_u} \min_u \sqrt{T_u/d_u} )$,
\fbqufur achieves a regret guarantee no worse than that of domain-aware uniform sampling, while being agnostic to $\cbr{(d_u, T_u)}_{u=1}^m$ and the domain memberships of the examples. For larger budget $B > \tilde{O}(\sum_u \sqrt{d_u T_u} \min_u \sqrt{T_u/d_u})$, \qufur's performance still matches the oracle baseline; we defer the discussion to Appendix~\ref{sec:large_budget}.

\subsection{Lower bound}
\label{sec:lower}

Our development so far establishes domain structure-aware regret upper bounds  $R=\tilde{O}(\tilde{\eta}^2 (\sum_u{\sqrt{d_u T_u}})^2/B)$, achieved by Fixed-Budget QuFUR and domain-aware uniform sampling baseline (the latter requires extra knowledge about the domain structure and domain membership of each example, whereas the former does not). 
In this section, we study optimality properties of the above upper bounds. Specifically, we show via Theorem~\ref{thm:lower} that they are tight up to logarithmic factors, for a wide range of domain structure specifications. Its proof can be found in Appendix~\ref{sec:proof_lower}.

\begin{theorem}
\label{thm:lower}
For any noise level $\eta \geq 1$, set of positive integers $\cbr{(d_u, T_u)}_{u=1}^m$ and integer $B$ that satisfy 
\begin{align} 
d_u &\leq T_u, \forall u \in [m],  \quad\sum_{u=1}^m d_u \leq d,  \nonumber \\
B &\geq \sum_{u=1}^m d_u, 
\label{eqn:lower-premise}
\end{align}
there exists an oblivious adversary such that:\\
1. It uses a ground truth linear predictor $\theta^\star \in \RR^d$ such that $\| \theta^* \|_2 \leq \sqrt{d}$, and for all $t$, $\abs{\inner{\theta^*}{x_t}} \leq 1$; in addition, the noises $\cbr{\xi_t}_{t=1}^T$ are sub-Gaussian with variance proxy $\eta^2$.\\
2. It shows example sequence $\cbr{x_t}_{t=1}^T \subset \cbr{x: \| x \|_2 \leq 1}$, such that $[T]$ can be partitioned into $m$ disjoint nonempty subsets $\cbr{I_u}_{u=1}^m$, where for each $u$, $|I_u|=T_u$, and  $\cbr{x_t}_{t \in I_u}$ lie in a subspace of dimension $d_u$.\\
3. Any online active learning algorithm $\cal{A}$ with label budget $B$ has regret $\Omega((\sum_{u=1}^m \sqrt{d_u T_u})^2/B)$.
\end{theorem}

The above theorem is a domain structure-aware refinement of the $\Omega(dT/B)$ minimax lower bound (Theorem~\ref{thm:lower-unstructured} in Appendix~\ref{sec:lower-unstructured}), in that it further constrains the adversary to present sequences of examples with domain structure parameterized by $\cbr{(d_u, T_u)}_{u=1}^m$. In fact, the $\Omega(dT/B)$ minimax lower bound is a special case of the lower bound of Theorem~\ref{thm:lower}, by taking $m = 1, d_1 = d$, and $T_1 = T$.

To discuss the tightness of the upper and lower bounds we obtained so far in more detail, we first set up some useful notations. Denote by $\E=\E(\cbr{(d_u, T_u)}_{u=1}^m)$ the set of oblivious adversaries that shows example sequences with domain structures specified by $\cbr{(d_u, T_u)}_{u=1}^m$. Additionally, denote by $\A(B)$ the set of online active learning algorithms that uses a label budget of $B$.
Finally, for an algorithm $\Acal$ and an oblivious adversary $\Ecal$, define $\R(\Acal, \Ecal)$ as the expected regret of $\Acal$ in the environment induced by $\Ecal$. 

Theorem~\ref{thm:lower} shows that for all $\cbr{(d_u, T_u)}_{u=1}^m$ and $B$ such that Eq.~\eqref{eqn:lower-premise} holds, we have
\[ 
   \min_{\Acal \in \A(B)} \; \max_{\Ecal \in \E}  \R\del{ \Acal, \Ecal}
   \geq 
   \Omega\left(\left(\sum_{u=1}^m \sqrt{d_u T_u}\right)^2/B\right).
\]

On the other hand, Theorem~\ref{thm:fixBudget} says for all $\cbr{(d_u, T_u)}_{u=1}^m$ and $B \le \tilde{O}\left(\sum_u \sqrt{d_u T_u} \min_{u \in [m]} \sqrt{T_u/d_u} \right)$, we have
\[
   \max_{\Ecal \in \E } \R\del{\qufur(B), \Ecal } 
   \leq 
   \tilde{O}\left(\left(\sum_{u=1}^m \sqrt{d_u T_u}\right)^2 / B\right).
\]

This shows that, for a wide range of domain structure specifications $\cbr{(d_u, T_u)}_{u=1}^m$ and budgets $B$ (i.e., $\sum_{u=1}^m d_u \leq B = \tilde{O}(\sum_u \sqrt{d_u T_u} \min_{u \in [m]} \sqrt{T_u/d_u})$), the regret guarantee of Fixed-Budget QuFUR is optimal; furthermore, the algorithm requires no knowledge on the domain structure. We call this property of Fixed-Budget QuFUR its {\em hidden-domain minimax optimality}.

\section{Extensions to realizable non-linear regression with an adaptive adversary}

\label{sec:nonlin}

\qufur's design principle, namely querying with probability proportional to uncertainty estimates of unlabeled data, can be easily generalized to deal with other active online learning problems. We demonstrate this by generalizing \qufur to non-linear regression with adaptive adversaries, using the concept of eluder dimension from~\citet{VR}.

In this section, we relax the assumption in Section~\ref{sec:linear} that domain structure is fixed before interaction starts --- we allow each input and its domain membership to depend on past history. Formally, we require the domain partition $\cbr{I_u: u \in [m]}$ to be {\em admissible}, defined as: 
\begin{definition}
\label{def-admissible}
The partition $\cbr{I_u: u \in [m]}$ is admissible, if the domain membership of the $t$-th example, $u_t \in [m]$ depends on the interaction history up to $t-1$ and unlabeled example $x_t$; formally, $u_t$ is $\sigma(H_{t-1}, x_t)$-measurable.
\end{definition}

\paragraph{Domain complexity measure.} Analogous to the dimension of the support in linear regression, we use $d_u'=\dim_u^E(\Fcal, 1/T_u^2)$, the \textit{eluder dimension} of $\Fcal$ with respect to domain $u \in [m]$ with support ${\Xcal}_u$,
to measure the complexity of a domain, formally defined below.
\begin{definition}
An input $x \in \Xcal$ is $\epsilon$-\textit{dependent} on a set of inputs $\{x_i\}_{i=1}^n \subseteq \Xcal$ with respect to $\Fcal$ if for all $f_1, f_2 \in \Fcal$, $\sqrt{\sum_{i=1}^n{(f_1(x_i)-f_2(x_i))^2}} \le \epsilon$ implies $f_1(x)-f_2(x) \le \epsilon$. 
$x$ is $\epsilon$-\textit{independent} of $\{x_i\}_{i=1}^n$ with respect to $\Fcal$ if it is not $\epsilon$-dependent on the latter.
\end{definition}

\begin{definition}
The $\epsilon$-\textit{eluder dimension} of $\Fcal$ with respect to support $\Xcal_u$, $\dim_u^E(\Fcal, \epsilon)$, is defined as the length of the longest sequence of elements in $\Xcal_u$ such that for some $\epsilon'>\epsilon$, every element is $\epsilon'$-independent of its predecessors.
\end{definition}

The above domain-dependent eluder dimension notion captures how effective the potential value of acquiring a new label can be estimated from labeled examples in domain $u$.\footnote{Appendix D in~\citet{VR} gives upper bounds of eluder dimensions for common function classes.}

\paragraph{The Algorithm.} The master algorithm, Algorithm~\ref{alg:4} in Appendix~\ref{sec:proof_nonlin}, runs $O(\log{T})$ copies of Algorithm~\ref{alg:3}. At round $t$, Algorithm~\ref{alg:3} predicts using the empirical square loss minimizer $\hat{f}_t$ based on all previously queried examples. 
Same as Algorithm~\ref{alg:1}, Algorithm~\ref{alg:3} queries with probability $\min\cbr{1, \alpha \Delta_t}$, where $\Delta_t$ is an uncertainty measure of the algorithm on example $x_t$. To compute the uncertainty measure, it constructs a confidence set $\mathcal{F}_t$, so that with high probability, $\mathcal{F}_t$ contains the ground truth $f^*$. The uncertainty measure $\Delta_t$ is the squared maximum disagreement on $x_t$ between two hypotheses in $\Fcal_t$. It can be shown that with high probability, the regret and query complexity are bounded by $ O(\sum_{t=1}^T{\Delta_t})$ and $O(\sum_{t=1}^T{\min\cbr{1, \alpha \Delta_t}})$, respectively.

\begin{algorithm}
\caption{QuFUR($\alpha$) for Nonlinear Regression \label{alg:3}}
\begin{algorithmic}[1]

\Require Hypothesis set $\Fcal$, time horizon $T$, parameters $\alpha, \delta, \eta$.
\State Labeled dataset $\Qcal\leftarrow \emptyset$.
\For {$t=1$ to $T$}
    \State Find $\hat{f}_t \leftarrow \argmin_{f \in \Fcal}{\sum_{i \in \Qcal}{(f(x_i)-y_i)^2}}$.
    \State Predict $\hat{f}_t(x_t)$.
    \State Define confidence set \[\Fcal_t \leftarrow \Bigg\{ f \in \Fcal: \sum_{i \in \Qcal}{(f(x_i)-\hat{f}_t(x_i))^2} \le \beta_{|\Qcal|}(\Fcal, \delta) \Bigg\},\] where $\beta$ is defined in Equation~\eqref{eq:beta_k}.
    \State Uncertainty $\Delta_t = \sup_{f_1, f_2 \in \Fcal_t}{\abs{f_1(x_t)-f_2(x_t)}^2}$.
    \State With probability $\min{\{1, \alpha \Delta_t\}}$, set $q_t=1$; otherwise set $q_t=0$.
    \If {$q_t=1$}
        \State Query $y_t$. $\Qcal \leftarrow \Qcal \bigcup \{t\}$.
    \EndIf
\EndFor

\end{algorithmic}
\end{algorithm}

We bound the regret of the algorithm on examples from domain $u$ in terms of domain complexity measure $R_u = \tilde{O}(\tilde{\eta}^2 d'_u \log{\Ncal(\Fcal, T^{-2},\| \cdot \|_\infty)})$,
where $\Ncal(\Fcal, \epsilon,\| \cdot \|_\infty)$ is the $\epsilon$-\textit{covering number} of $\Fcal$ with respect to  $\|\cdot\|_\infty$. Specifically, we prove the following theorem.

\begin{theorem}
\label{thm:nonlin}
Suppose the example sequence $\cbr{x_t}_{t=1}^T$ has the following structure: $[T]$ has an admissible partition $\cbr{I_u: u \in [m]}$, where for each $u$, $|I_u|=T_u$, and the eluder dimension of $\Fcal$ w.r.t. $\cbr{x_t}_{t \in I_u}$ is $d'_u$.
Then, given label budget $B \le \tilde{O}(\sum_u \sqrt{R_u T_u} \min_u \sqrt{R_u/T_u} )$, Algorithm~\ref{alg:4} satisfies:\\ 
1. It has query complexity $Q \leq B$; \\
2. With probability $1-\delta$, its regret  $R=\tilde{O}((\sum_u{\sqrt{R_u T_u}})^2/B).$
\end{theorem}

The proof of the theorem can be found in Appendix~\ref{sec:proof_nonlin}.
Specializing the theorem to linear hypothesis class $\Fcal =\{\inner{x}{ \theta}:\theta \in \RR^d, \|\theta\|_2 \le 1\}$, if $\Xcal_u$ is a subset of a $d_u$-dimensional subspace of $\RR^d$, we have $\dim_u^E(\Fcal, 1/T_u^2) =\tilde{O}(d_u)$ and $\log \Ncal(\Fcal, 1/T_u^2, \|\cdot\|_\infty)=\tilde{O}(d)$, implying $R_u = \tilde{O}( \tilde{\eta}^2 d_u d)$, which implies that $R =\tilde{O}( \tilde{\eta}^2 d (\sum_u{\sqrt{d_u T_u}})^2/B)$. Compared to Theorem~\ref{thm:fixBudget}, we conjecture that the additional factor $d$ is due to the increased difficulty with adaptive adversaries.

\section{Experiments}
We evaluate the query-regret tradeoffs of \qufur, the uniform query baseline (Section~\ref{sec:baselines}), and naive greedy query (i.e., always query until labeling budget is exhausted) on two linear regression and two classification tasks. Although \qufur is designed for regression, experiments show that the same query strategy also achieves competitive performance on high-dimensional multi-class classifications tasks. See Appendix~\ref{sec:exp_details} for more details.

\begin{figure}[ht]
\begin{center}
\includegraphics[width=0.49\columnwidth]{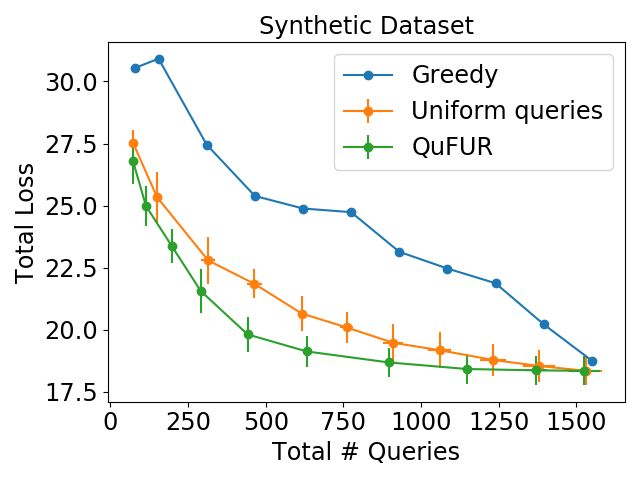}
\includegraphics[width=0.49\columnwidth]{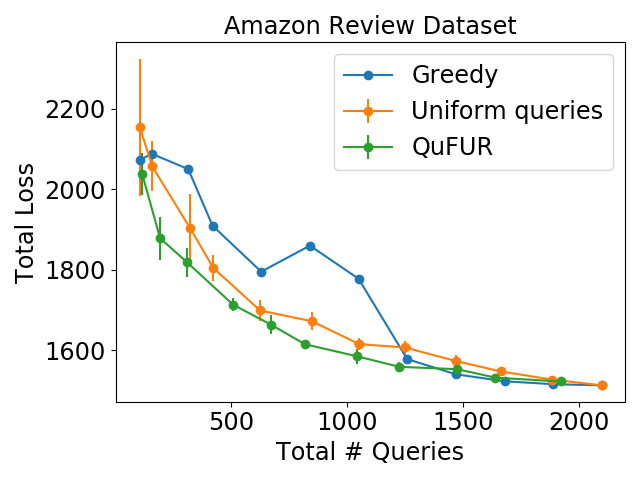}
\includegraphics[width=0.49\columnwidth]{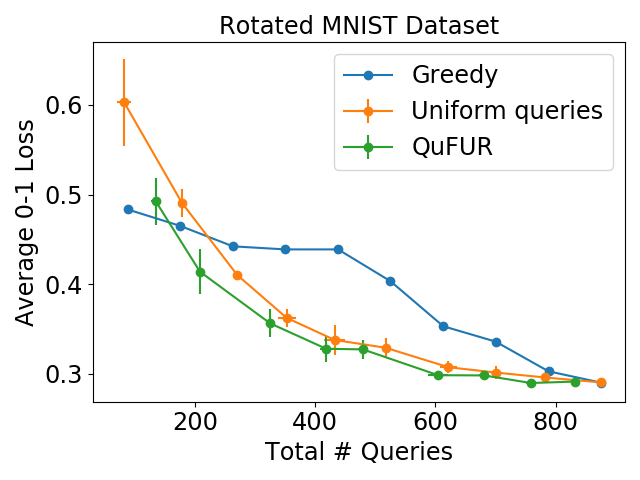}
\includegraphics[width=0.49\columnwidth]{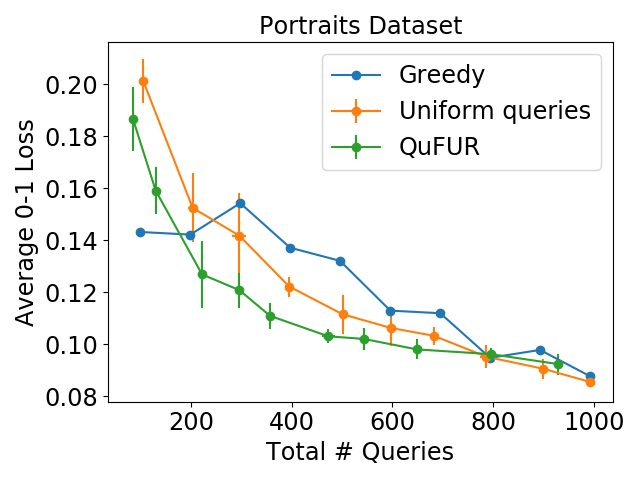}
\caption{Total squared loss (for regression tasks) / average 0-1 error (for classification tasks) vs. total number of queries in synthetic, Amazon reviews, rotated MNIST, and portraits datasets. Error bars show stddev across 5 runs. QuFUR has the best tradeoffs.} 
\label{fig:comparison}
\end{center}
\end{figure}

\textbf{Synthetic dataset} is a regression task where the label is generated via $y_t = x_t^\top \theta^*+\xi_t$. Inputs $x_i$'s come from 20 domains that are orthogonal linear subspaces with $d=88$. Each domain $u$ has either $T_u=100$ and $d_u=6$, or $T_u=50$ and $d_u=3$. $\theta^*$ is a random vector on the unit sphere in $\RR^d$. For any $x_i$ in domain $u$, $x_i = V_u z_i$ where $V_u$ is an orthonormal basis of $\Xcal_u$, and $z_i$ is drawn from the unit sphere in $\RR^{d_u}$. Noise $\xi_t$'s are iid zero-mean Gaussian with variance $\eta^2=0.1$.

\textbf{Amazon review dataset}~\citep{mcauley2015image} is a regression task where we predict ratings from 1 to 5 based on review text. Reviews come from 3 topics / domains: automotive, grocery, video games. We assume that the domains come in succession, with durations $300, 600, 1200$. Each review is encoded as a 768-dimensional vector --- the average BERT embedding~\citep{devlin2018bert} of each word in the review. Each domain uses a subset of the vocabulary, so the embeddings within the domain reside in a subspace ~\citep{Reif}. The sub-vocabularies are of smaller size (and largely disjoint),  motivating our low-dimensional subspace structure for linear models. To check that realizability is a reasonable assumption, we verify that offline linear regression on all domains achieve an MSE of 0.62. 

\begin{figure}[ht]
\begin{center}
\includegraphics[width=\columnwidth]{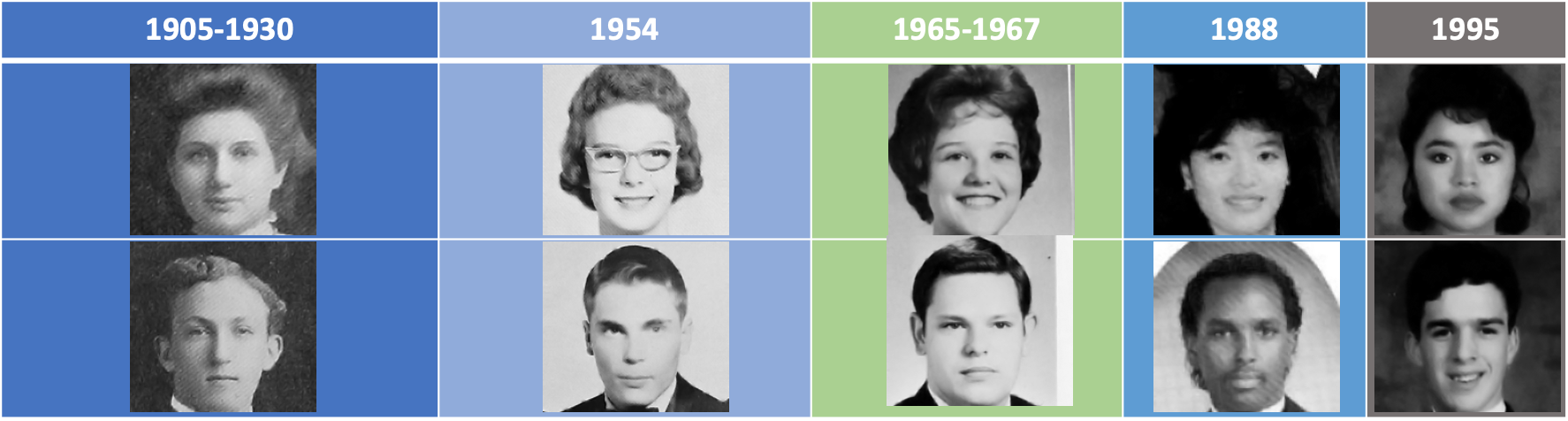}
\caption{Illustration of modified portraits dataset. Block lengths (not drawn to scale) indicate domain durations, which are unknown to the learner. Facial features for each gender shift over time.
} 
\label{fig:portaits_ilt}
\end{center}
\vspace{-0.5cm}
\end{figure}

\textbf{Rotated MNIST dataset}~\citep{MNIST} is a 10-way classification task. We create 3 domains via rotating the images 60, 30, and 0 degrees. Domain durations are $500, 250$, and $125$ in Figure~\ref{fig:comparison}.
We check that a linear classifier trained on all domains obtain 100\% training accuracy.

\textbf{Portraits dataset}~\citep{ginosar2015century, kumar2020understanding} contains photos of high school
seniors taken across 1900s-2010s, labeled by gender. We sort the photos chronologically, and divide into 5 periods with 8000 photos each. We pick the first $\{512, 256, 128, 64, 32\}$ photos from each period to obtain 5 domains (Figure~\ref{fig:portaits_ilt}). We check that a linear classifier trained on all domains obtain 99\% training accuracy.

\begin{figure}[ht]
\begin{center}
\includegraphics[width=0.32\columnwidth]{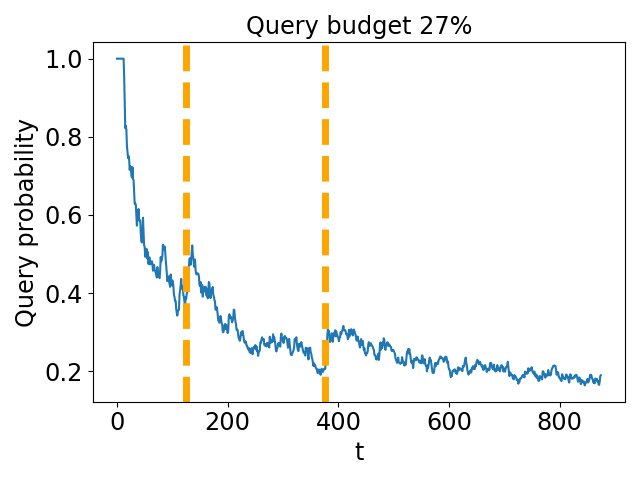}
\includegraphics[width=0.32\columnwidth]{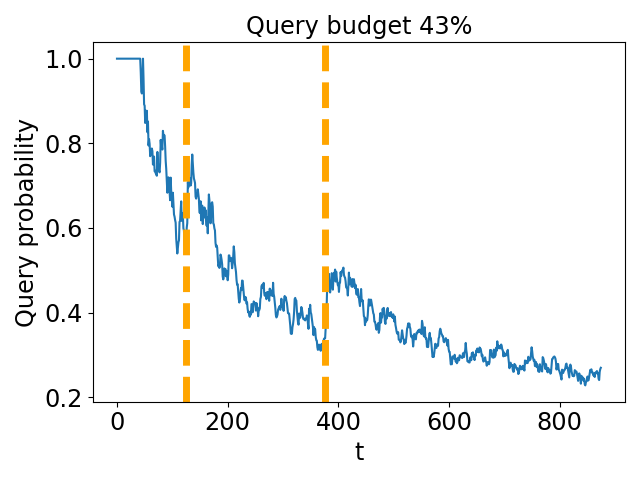}
\includegraphics[width=0.32\columnwidth]{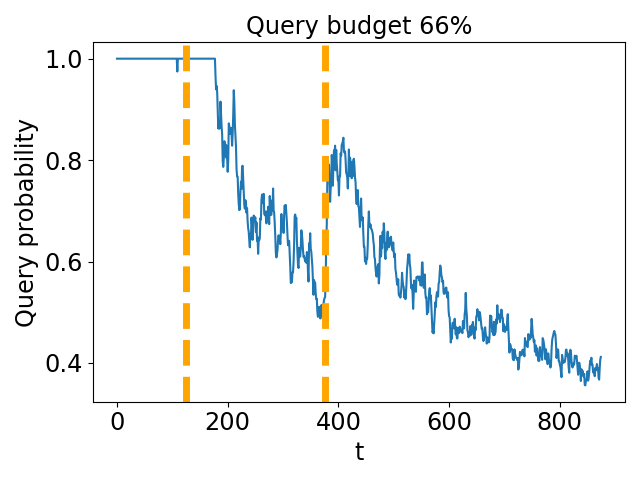}
\caption{\qufur's query probability in the rotated MNIST experiment with domain durations 125 / 250 / 500, when alpha is set to 0.25 (\textbf{left}), 0.5 (\textbf{middle}), 1 (\textbf{right}). \qufur queries more frequently upon domain shift.} 
\label{fig:query_prob}
\end{center}
\vspace{-0.45cm}
\end{figure}

\paragraph{Results.} We run QuFUR($\alpha$) for $\alpha$ sweeping an appropriate range for each dataset, and uniform queries with probability $\mu \in [0.05, 1]$. Figure~\ref{fig:comparison} shows that QuFUR achieves the lowest total regret under the same labeling budget across all datasets. Figure~\ref{fig:query_prob} shows that \qufur's query probability abruptly rises upon domain shifts. We choose highly heterogeneous domain durations since our theory predicts that \qufur is likely to have the most savings in such situations. We show in Appendix~\ref{sec:exp_details} that in other setups, \qufur still has competitive performance.
\section{Conclusion}

We formulate a novel task of active online learning with latent domain structure. We propose a surprisingly simple algorithm that adapts to domain shifts, and give matching upper and lower bounds in a wide range of domain structure specifications for linear regression. The strategy is readily generalizable to other problems, as we did for non-linear regression, simply relying on a suitable uncertainty estimate for unlabeled data. We believe that our problem and solution can spur future work on making active online learning more practical.

\section*{Acknowledgements}
We would like to thank Alina Beygelzimer, Akshay Krishnamurthy, Rob Schapire and Steven Wu for helpful discussions on active learning in the online setting and thank Chelsea Finn and Sergey Levine for discussing the motivation of the proposed questions. YC is supported by Stanford Graduate Fellowship. HL is supported by NSF Awards IIS-1755781 and IIS-1943607. TM is partially supported by the Google Faculty Award, Stanford Data Science Initiative, Stanford Artificial Intelligence Laboratory, and Lam Research. CZ acknowledges the startup funding from the University of Arizona for support.

\bibliographystyle{plainnat}
\bibliography{references}

\newpage
\appendix
\section{Proofs for upper bounds}

\subsection{Proof of Theorem~\ref{thm:upper}}
\label{sec:proof_upper}

We provide the proof of Theorem~\ref{thm:upper} in this section. We focus on regret and query complexity bounds on one domain $I_u$, and sum over domain $u$ to obtain Theorem~\ref{thm:upper}. Recall that we define the interaction history between the learner and the environment up to time $t$ be $H_t = \{x_{1:t}, f_{1:t}, \xi_{1:t}\}$; we abbreviate $\EE[\cdot|x_t, H_{t-1}]$ as $\EE_{t-1}[\cdot]$. 

The following lemma upper bounds the regret with sum of uncertainty estimates, $\Delta_t = \tilde{\eta}^2 \min\del{1,  \| x_t \|_{M_t^{-1}}^2}$. A similar lemma has appeared in~\citet[Lemma 1]{Bianchi}.

\begin{lemma}
\label{lem:subseq}
In the setting of Theorem~\ref{thm:upper}, with probability $1-\frac{\delta}{2}$, for all $t \in [T]$,
$(\hat{y}_t - \inner{\theta^*}{x_t})^2  = \tilde{O}\del{{\Delta_t}}$.
\end{lemma}

\begin{proof}[Proof of Lemma~\ref{lem:subseq}]

Denote the value of $M, \mathcal{Q}$ at the beginning of round $t$ as $M_t, \mathcal{Q}_t$. Let $\lambda=1/C^2$, $V_t = M_t-\lambda I= \sum_{s \in \Qcal_t} x_s x_s^\top$. Therefore, $\hat{\theta}_t = M_t^{-1} (\sum_{s \in \Qcal_t} x_s y_s) = M_t^{-1}(V_t \theta^* + \sum_{s \in \Qcal_t} \xi_s x_s)$, and
\begin{equation}
\inner{x_t}{\hat{\theta}_t - \theta^*} = \sum_{s \in \Qcal_t} \xi_s (x_t^\top M_t^{-1} x_s) - \lambda x_t^\top M_t^{-1} \theta^*.
\label{eqn:bias-var-decomp}
\end{equation}
The first term is a sum over a set of independent sub-Gaussian random variables, so it is $(\eta\sigma)^2$-sub-Gaussian with $\sigma^2 = \sum_{s \in \Qcal_t} x_t^\top M_t^{-1} x_s x_s^\top M_t^{-1} x_t \leq x_t^\top M_t^{-1} x_t$. Define event \[ E_t = \cbr{ \abs{\sum_{s \in \Qcal_t} \xi_s (x_t^\top M_t^{-1} x_s) } \leq \eta \sqrt{ 2 \ln{\frac{4 T}\delta} } \| x_t \|_{M_t^{-1}}  }. \]
By standard concentration of subgaussian random variables, we have $\PP(E_t) \geq 1 - \frac{\delta}{2T}$. Define $E = \cap_{t=1}^T E_t$. By union bound, we have $\PP(E) \geq 1 - \frac{\delta}{2}$. We henceforth condition on $E$ happening, in which case the first term of Equation~\eqref{eqn:bias-var-decomp} is bounded by $\eta \sqrt{ 2\ln{(4T/\delta)} } \| x_t \|_{M_t^{-1}}$ at every time step $t$.
    
Meanwhile, the second term of Equation~\eqref{eqn:bias-var-decomp} can be bounded by Cauchy-Schwarz:
\[ 
\abs{\lambda x_t^\top M_t^{-1} \theta^*} = \lambda \abs{\inner{ M_t^{-1/2} x_t }{ M_t^{-1/2} \theta^*}} \leq \lambda \| x_t \|_{M_t^{-1}} \| \theta^* \|_{M_t^{-1}} 
\leq \sqrt{\lambda} \| \theta^* \|_2 \| x_t \|_{M_t^{-1}},
\] which is at most $\| x_t \|_{M_t^{-1}}$, since $\| \theta^* \|_2 \leq C$ and $\lambda = 1/C^2$. 
Using the basic fact that $(A + B)^2 \leq 2A^2 + 2B^2$,
\[ ( \inner{x_t}{\hat{\theta}_t} - \inner{x_t}{\theta^*} )^2 \leq  (4 \eta^2 \ln{(4T/\delta)}+2) \| x_t \|_{M_t^{-1}}^2. \]
Since $\hat{y}_t = \clip(\inner{x_t}{\hat{\theta}_t} ) \in [-1, 1]$ and $\abs{\inner{x_t}{\theta^*}} \le 1$, we also trivially have $(\hat{y}_t - \inner{\theta^*}{x_t})^2 \leq 4$. Therefore,
    \begin{align*}
       (\hat{y}_t- \inner{\theta^*}{x_t})^2 
       &\leq
       \min\del{4, (4\eta^2 \ln{(4T/\delta')}+2) \| x_t \|_{M_t^{-1}}^2} \\
       &\leq
        (4 \eta^2 \ln{(2T/\delta')}+4) \cdot \min \del{1, \| x_t \|_{M_t^{-1}}^2} \\
       &\leq \tilde{O} \del{ \tilde{\eta}^2 \min \del{1, \| x_t \|_{M_t^{-1}}^2} } = \tilde{O}(\Delta_t). \qedhere
    \end{align*} 
\end{proof}

The following lemma bounds the sum of uncertainty estimates for $k$ \textit{queried} examples in a domain:
\begin{lemma}
\label{lem:logdet}
    Let $a_1, \dots, a_k$ be $k$ vectors in $\RR^d$. For $i \in [k]$, define $N_i = \lambda I+\sum_{j=1}^{i-1}{a_j a_j^\top}$.
    Then, for any $S \subseteq [k]$, 
    $\sum_{i \in S} \min\del{1, \|a_i\|_{N_i^{-1}}^2} \leq \ln(\det(\lambda I + \sum_{i \in S}{a_i a_i^\top} )/\det(\lambda I))$.
\end{lemma}

\begin{proof}[Proof of Lemma~\ref{lem:logdet}]
We denote by $N_{i,S} = \lambda I+ \sum_{j \in S: j \leq i-1} {a_j a_j^\top}$. As $S$ is a subset of $[k]$, we have that $N_{i,S} \preceq N_i$. Consequently, $\|a_i\|_{N_i^{-1}} \leq \| a_i \|_{N_{i,S}^{-1}}$. Therefore,
\[ \sum_{i \in S} \min\del{1, \|a_i\|_{N_i^{-1}}^2}  
   \leq  \sum_{i \in S} \min\del{1, \|a_i\|_{N_{i,S}^{-1}}^2} 
   \leq \ln\del{\frac{\det(\lambda I + \sum_{i \in S}{a_i a_i^\top} )}{\det(\lambda I)}}, \]
 where the second inequality is well-known~\citep[see e.g.][ Lemma 19.4]{lattimore2018bandit}.
\end{proof}

\begin{proof}[Proof of Theorem~\ref{thm:upper}]
Let $p_t = \min(1, \alpha \Delta_t)$ be the learner's query probability at time $t$; it is easy to see that $\EE_{t-1} \sbr{q_t} = p_t$.

Let random variable $Z_t = q_t \Delta_t$. We have the following simple facts: 
\begin{enumerate}
    \item $Z_t \le \tilde{\eta}^2$,
    \item $\EE_{t-1} Z_t = p_t \Delta_t$,
    \item $\EE_{t-1} Z_t^2 \leq \tilde{\eta}^2 \cdot \EE_{t-1} Z_t \leq  \tilde{\eta}^2 p_t \Delta_t$.
\end{enumerate}
For every $u \in [m]$, define event 
\begin{equation}
F_u = 
\cbr{
\abs{ \sum_{t \in I_u} p_t \Delta_t - \sum_{t \in I_u} q_t \Delta_t }
\leq  O\left( \tilde{\eta}
\sqrt{\sum_{t \in I_u} p_t \Delta_t \ln\frac{T}{\delta} }
+  \tilde{\eta}^2 \ln\frac{T}{\delta} \right)
}.
\label{eqn:lc-freedman}
\end{equation}
Applying Freedman's inequality to $\cbr{Z_t}_{t \in I_u}$~\citep[see e.g.][Lemma 2]{bartlett2008high}, we have that $\PP(F_u) \geq 1- \frac{\delta}{4m}$. 

Similarly, define 
\begin{equation}
G = 
\cbr{
\abs{ \sum_{t=1}^T p_t - \sum_{t=1}^T q_t }
\leq  O\left( 
\sqrt{\sum_{t=1}^T p_t \ln\frac{T}{\delta} }
+ \ln\frac{T}{\delta} \right)
}.
\label{eqn:lc-freedman-2}
\end{equation}
Applying Freedman's inequality to $\cbr{q_t}_{t \in I_u}$, we have that $\PP(G) \geq 1- \frac{\delta}{4}$. 

Furthermore, define $H = E \cap (\cap_{u=1}^m F_u) \cap G $, where $E$ is the event defined in the proof of Lemma~\ref{lem:subseq}. By union bound, $\PP(H) \geq 1 - \delta$. We henceforth condition on $H$ happening. 

By the definition of $F_u$, 
Solving for $\sum_{t \in I_u} p_t \Delta_t$ in Equation~\eqref{eqn:lc-freedman}, we get that
\begin{equation}
\sum_{t \in I_u} p_t \Delta_t 
= \tilde{O}\left( \sum_{t \in I_u} q_t \Delta_t + \tilde{\eta}^2 \right). 
\label{eqn:lc-solved}
\end{equation}

Using Lemma~\ref{lem:logdet} with $\{a_i\}_{i=1}^k = \cbr{x_t}_{t \in \Qcal_T}$, and $S = I_u \cap \Qcal_T$, we get that 
\begin{align*}
  \sum_{t \in I_u} q_t \Delta_t 
  & \le
  \tilde{\eta}^2 \cdot \ln\det\del{ I + C^2 \sum_{t \in I_u \cap \Qcal_T} x_t x_t^\top } \\
  & \le 
  2\tilde{\eta}^2 d_u \ln{\left(1+C^2 T_u/d_u\right)}  = \tilde{O}(\tilde{\eta}^2 d_u ).
\end{align*}
In combination with Equation~\eqref{eqn:lc-solved}, we have $\sum_{t \in I_u} p_t \Delta_t =\tilde{O}(\tilde{\eta}^2 d_u )$.

We divide the examples in domain $u$ into high and low risk subsets with index sets $I_{u,+}$ and $I_{u,-}$ (abbrev. $I_+$ and $I_-$ hereafter). Formally, 
\[ I_+=\{t \in I_u: \alpha \Delta_t>1 \}, \quad I_-=I-I_+.\]

We consider bounding the regrets and the query complexities in these two sets respectively:
\begin{enumerate}
\item For every $t$ in $I_+$, as $p_t = 1$, label $y_t$ is queried, so
\[
\sum_{t \in I_+}{\Delta_t} 
= \sum_{t \in I_+}{q_t \Delta_t}
\leq \sum_{t \in I_u} q_t \Delta_t 
=\tilde{O}(\tilde{\eta}^2 d_u).
\]
Since for every $t$ in $I_-$, $\Delta_t > 1/\alpha$, we have $\sum_{t \in I_+}{\Delta_t}  > |I_+|/\alpha$. This implies that $\sum_{t \in I_+} p_t = |I_+| =\tilde{O}( \alpha \tilde{\eta}^2 d_u)$.

\item For every $t$ in $I_-$, $p_t = \alpha \Delta_t$. Therefore, $\sum_{t \in I_-}{\alpha \Delta^2_t } = \sum_{t \in I_-}{p_t \Delta_t } \leq \sum_{t \in I_u}{p_t \Delta_t } = \tilde{O}(\tilde{\eta}^2 d_u)$. By Cauchy-Schwarz, and the fact that $\abs{I_-} \leq T_u$, we get that $\sum_{t \in I_-}{\Delta_t} \leq \sqrt{\abs{I_-} \cdot (\sum_{t \in I_-}{\Delta^2_t }) }  = \tilde{O}(\tilde{\eta} \sqrt{d_u T_u/\alpha})$. 

Consequently, $\sum_{t \in I_-} p_t= \sum_{t \in I_-}{\alpha \Delta_t} \leq \tilde{O}(\tilde{\eta} \sqrt{\alpha d_u T_u})$. 

\end{enumerate}

Summing over the two cases, we have
\[ 
   \sum_{t \in I_u} p_t
   \leq
   \otil{\alpha \tilde{\eta}^2 d_u + \tilde{\eta} \sqrt{\alpha d_u T_u}},
   \quad
   \sum_{t \in I_u} \Delta_t 
   \leq 
   \otil{ \tilde{\eta}^2 d_u + \tilde{\eta} \sqrt{d_u T_u/\alpha}},
\]
If $\alpha \le \frac1{\tilde{\eta}^2} \frac{T_u}{d_u}$, we have $\alpha \tilde{\eta}^2 d_u \leq \tilde{\eta} \sqrt{\alpha d_u T_u}$, otherwise we use the trivial bound $\sum_{t \in I_u} p_t \leq T_u$. Therefore, the above bounds can be simplified to
\begin{equation}
   \sum_{t \in I_u} p_t
   \leq
   \otil{\min\{T_u, \tilde{\eta} \sqrt{\alpha d_u T_u}\}},
   \quad
   \sum_{t \in I_u} \Delta_t 
   \leq 
   \otil{\max\{\tilde{\eta}^2 d_u,\tilde{\eta} \sqrt{d_u T_u/\alpha}\}}.
   \label{eqn:grt-exp}
\end{equation}
For the query complexity, from the definition of event $G$, applying AM-GM inequality on Equation~\eqref{eqn:lc-freedman-2}, we also have 
\[ 
Q = \sum_{t=1}^T q_t 
= \tilde{O}\left( \sum_{t=1}^T p_t + 1 \right)
= \tilde{O}\left(  \sum_{u=1}^m \min\{T_u, \tilde{\eta} \sqrt{\alpha d_u T_u}\} + 1 \right). 
\]

For the regret guarantee, we have by the definition of event $E$ and Lemma~\ref{lem:subseq} that
\[
\sum_{t=1}^T (\hat{y}_t - \inner{\theta^*}{x_t})^2 
\leq
\otil{ \sum_{t=1}^T \Delta_t^2}
=
\otil{ \sum_{u=1}^m \del{ \sum_{t \in I_u} \Delta_t^2 } }.
\]
Using the second inequality of Equation~\eqref{eqn:grt-exp}, we get
\[ 
\sum_{t=1}^T (\hat{y}_t - \inner{\theta^*}{x_t})^2 
\leq
\tilde{O}\del{ \sum_{u=1}^m \max\{\tilde{\eta}^2 d_u,\tilde{\eta} \sqrt{d_u T_u/\alpha}\}} .
\]
The theorem follows. 
\end{proof}

\subsection{Proof of Theorem~\ref{thm:fixBudget}}
\label{sec:proof_fixBudget}

Before going into the proof, we set up some useful notations. 
Define $I = \cbr{0,1,\ldots,k}$ as the index set of the $\alpha_i$'s of interest. 
Recall the number of copies $k = 1+ \lceil 3\log T \rceil \leq 2 + 3\log T$. Recall also that $B' = \lfloor B / k \rfloor$ is the label budget for each copy.

Let $p_t^i=\min(1, \alpha_i \Delta_t)$ be the intended query probability of copy $i$ at time step $t$; 
let $r^i_t \sim {\rm Bernoulli}(p_t^i)$ be the attempted query decision of copy $i$ at time step $t$;
let $A^i_t = \one\sbr{\sum_{j=1}^{t-1}{r^i_j} < B'}$, i.e. the indicator that copy $i$ has not reached its budget limit at time step $t$. 
Using this notation, the actual query decision of copy $i$, $q_t^i$, can be written as $r_t^i A_t^i$.

We have the following useful observation that gives a sufficient condition for copy $i$ to be within its label budget:
\begin{lemma}
Given $i \in [k]$, if $\sum_{t=1}^T{A_t^i r^i_t} <B'$, the following items hold:
\begin{enumerate}
    \item $\sum_{t=1}^T{r^i_t} <B'$.
    \item For all $t \in [T]$, $A_t^i = 1$, i.e. copy $i$ does not run of label budget throughout.
\end{enumerate}
\label{lem:trick}
\end{lemma}
\begin{proof}
Suppose for the sake of contradiction that $\sum_{t=1}^T{r^i_t} \geq B'$.
Consider the first $B'$ occurrences of $r_j^i = 1$; call them $J = \cbr{j_1, \ldots, j_{B'}}$. It can be seen that for all $j \in J$, $A_j^i = 1$. Therefore,
\[ 
\sum_{t=1}^T{A_t^i r^i_t}
\geq
\sum_{j \in J} {A_j^i r_j^i}
\geq 
\abs{J} = B',
\]
which contradicts with the premise that $\sum_{t=1}^T{A_t^i r^i_t} <B'$.

The second item immediately follows from the first item, as $\sum_{j=1}^{T}{r^i_j} < B'$ implies that $\sum_{j=1}^{t-1}{r^i_j} < B'$ for every $t \in [T]$.
\end{proof}

Complementary to the above lemma, we can also see that for every $i \in [k]$, $\sum_{t=1}^T{A_t^i r^i_t} = \sum_{t=1}^T q_t^i \leq B'$ is trivially true.
We next give a key lemma that generalizes Theorem~\ref{thm:upper}, and upper bounds $\sum_{t=1}^T{A_t^i r^i_t}$ for all $i$'s beyond the above trivial $B'$ bound. 

\begin{lemma} \label{lem:apriori}
There exists $C = \polylog(T,\frac1\delta) \geq 1$, such that
with probability $1-\delta/2$,
\[
\sum_{t=1}^T{A_t^i \Delta_t} \leq C \cdot \tilde{\eta} \sum_u{\sqrt{d_u T_u}/\sqrt{\alpha_i}}, \text{ and } 
\sum_{t=1}^T{A_t^i r_t^i} \leq C \cdot \tilde{\eta}\sqrt{\alpha_i}\sum_u{\sqrt{d_u T_u}},
\]
for every $i \in I$ such that 
$\alpha_i \in \intcc{ \frac{1}{\tilde{\eta}^2} \del{\frac{1}{\sum_u \sqrt{d_u T_u}}}^2, \frac{1}{\tilde{\eta}^2} \min_{u \in [m]} \frac{T_u}{d_u}}$.
\end{lemma}
\begin{proof}
 Applying Freedman's inequality to the martingale difference sequence $\{A_t^i(r_t^i-p_t^i) \}_{t=1}^T$, we get that with probability $1-\delta/4$,
\begin{align}
\sum_{t=1}^T {A_t^i r_t^i} = \tilde{O}\left(\sum_{t=1}^T {A_t^i p_t^i} + 1 \right).
    \label{freedman_eq1}
\end{align}

Applying Freedman's inequality to $\{A_t^i(r_t^i-p_t^i)\Delta_t \one[t \in I_u]\}_{t=1}^T$, and take a union bound over all $u \in [m]$, we get that with probability $1-\delta/4$,
\[
\sum_{t \in I_u}{A_t^i p_t^i \Delta_t} = \tilde{O}\left(\sum_{t \in I_u}{A_t^i r_t^i \Delta_t}+\tilde{\eta}^2\right).\]
Using Lemma~\ref{lem:logdet} with $\{a_i\}_{i=1}^k = \cbr{x_t}_{t \in \Qcal_T}$, and $S = I_u \cap \Qcal_T$ we get that, deterministically, $\sum_{t \in I_u}{A_t^i r_t^i \Delta_t} \le \sum_{t \in I_u}{q_t \Delta_t}=\tilde{O}(\tilde{\eta}^2 d_u)$. So with probability $1-\delta/4$,
\begin{align}
\sum_{t \in I_u}{A_t^i p_t^i \Delta_t} = \tilde{O}(\tilde{\eta}^2 d_I).
\label{freedman_eq2}
\end{align}

We henceforth condition on Equations~\eqref{freedman_eq1} and~\eqref{freedman_eq2} occuring, which happens with probability $1-\delta/2$ by union bound. 
Let $I_+=\{t \in I_u: \alpha_i \Delta_j>1\}$, and $I_-=I_u-I_+$.

\begin{enumerate}
\item For $I_+$, by Equation~\eqref{freedman_eq2}, $
\sum_{t \in I_+}{A_j^i \Delta_j}  =\tilde{O}(\tilde{\eta}^2 d_u) \implies \sum_{j \in I_+}{A_j^i p_j^i} =\tilde{O}( \alpha_i \tilde{\eta}^2 d_u)$.

\item For $I_-$, by Equation~\eqref{freedman_eq2}, $\sum_{j \in I_-}{A_j^i \alpha_i \Delta^2_j } = \sum_{j \in I_-}{A_j^i p_j \Delta_j } =\tilde{O}(\tilde{\eta}^2 d_u)$; this implies that $\sum_{j \in I_-}{A_j^i \Delta_j} = \tilde{O}(\tilde{\eta} \sqrt{d_u T_u/\alpha_i})$. 
In this event, we also have 
$\sum_{j \in I_-}{A_j^i p_j^i} 
= \sum_{j \in I_-}{A_j^i \alpha_i \Delta_j} = \tilde{O}(\tilde{\eta} \sqrt{d_u T_u \alpha_i})$.
\end{enumerate}

Summing over the two cases, we have
\[ 
   \sum_{t \in I_u} A_t^i p_t^i
   \leq
   \tilde{O}(\alpha_i \tilde{\eta}^2 d_u + \tilde{\eta} \sqrt{\alpha_i d_u T_u}),
   \quad
   \sum_{t \in I_u} A_t^i \Delta_t 
   \leq 
   \tilde{O}(\tilde{\eta}^2 d_u + \tilde{\eta} \sqrt{d_u T_u/\alpha_i}),
\]
By the assumption that $\alpha_i \le \frac1{\tilde{\eta}^2} \min_{u} \frac{T_u}{d_u}$, for every $u$, we have, $\alpha_i \tilde{\eta}^2 d_u \leq \tilde{\eta} \sqrt{\alpha_i d_u T_u}$. This implies that 
\begin{align}
   \sum_{t \in I_u} A_t^i p_t^i
   \leq
   \tilde{O}(\tilde{\eta} \sqrt{\alpha_i d_u T_u}),
   \quad
   \sum_{t \in I_u} A_t^i \Delta_t 
   \leq 
   \tilde{O}(\tilde{\eta} \sqrt{d_u T_u/\alpha_i}).
\end{align}
Summing over $u \in [m]$, we get
\[
    \sum_{t=1}^T A_t^i p_t^i
   \leq
   \tilde{O}(\tilde{\eta} \sum_{u=1}^m \sqrt{\alpha_i d_u T_u}),
   \quad
   \sum_{t=1}^T A_t^i \Delta_t 
   \leq 
   \tilde{O}(\tilde{\eta} \sum_{u=1}^m \sqrt{d_u T_u/\alpha_i}).
\]
Therefore, using Equation~\eqref{freedman_eq1}, we have
\[
   \sum_{t=1}^T A_t^i r_t^i 
   \leq 
   \tilde{O}\left(\sum_{t=1}^T {A_t^i p_t^i} + 1 \right)
   \leq
   \tilde{O}\del{ \tilde{\eta} \sum_{u=1}^m \sqrt{\alpha_i d_u T_u} + 1} 
   \leq 
   \tilde{O}\del{ \tilde{\eta} \sum_{u=1}^m \sqrt{\alpha_i d_u T_u}},
\]
where the last inequality uses the assumption that $\alpha_i \geq \frac{1}{\tilde{\eta}^2} \del{\frac{1}{\sum_u \sqrt{d_u T_u}}}^2$.
The lemma follows.
\end{proof}

We are now ready to prove Theorem~\ref{thm:fixBudget}.
\begin{proof}[Proof of Theorem~\ref{thm:fixBudget}]
First, the query complexity of \fbqufur is $B$ by construction, as the algorithm maintains $k$ copies of \qufur, and each copy consumes at most $B' = \lfloor B/k \rfloor$ labels.

We now bound the regret of \fbqufur. 
We consider $\wbar{B} = C k (\sum_u \sqrt{d_u T_u}) \cdot \min_{u \in [m]} \sqrt{T_u/d_u} = \otil{ (\sum_u \sqrt{d_u T_u}) \cdot \min_{u \in [m]} \sqrt{T_u/d_u} }$, where $C = \polylog(T,\frac1\delta) \geq 1$ is defined in Lemma~\ref{lem:apriori}.
We will show that if $B \in (0,\wbar{B}]$, with probability $1-\delta$, the regret of \fbqufur is at most $\tilde{O}\del{ \frac{\tilde{\eta}^2 ( \sum_u\sqrt{d_u T_u})^2 }{B} }$.

If $B <2C \tilde{\eta}^2 k$, the regret of the algorithm is trivially upper bounded by $4T$, which is clearly $\tilde{O}\del{ \frac{ \tilde{\eta}^2 ( \sum_u\sqrt{d_u T_u})^2 }{B} }$.
Therefore, throughout the rest of the proof, we consider $B \in [2C \tilde{\eta}^2 k, \wbar{B}]$.

Recall that $I = \cbr{\frac{2^i}{T^2}: i \in \cbr{0,1,\ldots,k}}$. We denote by $\alpha_{\min} = \frac{1}{T^2}$ the minimum element of $I$, and $\alpha_{\max} = \frac{2^k}{T^2} \geq T$ the maximum element of $I$.

Denote by
\[ 
i_B 
= 
\max\cbr{ i \in I: C \tilde{\eta} \sqrt{\alpha_i} \sum_{u=1}^m \sqrt{d_u T_u} < B'}
=
\max\cbr{ i \in I: 
\alpha_i < \del{\frac{B'}{C \tilde{\eta} \sum_u\sqrt{d_u T_u} }}^2
}.
\]
As $B \in [2C\tilde{\eta}^2 k, \wbar{B}]$, we have 
$\del{\frac{B'}{C \tilde{\eta} \sum_u\sqrt{d_u T_u} }}^2 \in (\alpha_{\min}, \alpha_{\max}]$. 
Indeed, $\del{\frac{B'}{C \tilde{\eta} \sum_u\sqrt{d_u T_u} }}^2 \leq \del{\frac{\wbar{B}}{C k \tilde{\eta} \sum_u\sqrt{d_u T_u} }}^2 \leq T \leq \alpha_{\max}$, 
$\del{\frac{B'}{C \tilde{\eta} \sum_u\sqrt{d_u T_u} }}^2 \geq \del{\frac{\tilde{\eta}}{\sum_u\sqrt{d_u T_u} }}^2 > \alpha_{\min}$, as $\sum_u \sqrt{d_u T_u} \leq \sum_u T_u = T$.

Therefore, by the definition of $i_B$, we have
\begin{equation}
\alpha_{i_B} \in \intco{ \frac12 \del{\frac{B'}{C \tilde{\eta} \sum_u\sqrt{d_u T_u} }}^2, \del{\frac{B'}{C \tilde{\eta} \sum_u\sqrt{d_u T_u} }}^2}
\label{eqn:ib}
\end{equation}
Again by our assumption on $B$, $\frac12 \del{\frac{B'}{C \tilde{\eta} \sum_u\sqrt{d_u T_u} }}^2 \geq
\tilde{\eta}^2 \del{\frac{1}{\sum_u \sqrt{d_u T_u}}}^2
\geq \frac{1}{\tilde{\eta}^2} \del{\frac{1}{\sum_u \sqrt{d_u T_u}}}^2$, 
$\del{\frac{B'}{C \tilde{\eta} \sum_u\sqrt{d_u T_u} }}^2 \leq \del{\frac{\wbar{B}}{C k \tilde{\eta} \sum_u\sqrt{d_u T_u} }}^2 \leq \frac{1}{\tilde{\eta}^2} \min_{u \in [m]} \frac{T_u}{d_u}$.
Therefore,
\[
\alpha_{i_B} \in \intcc{ \frac{1}{\tilde{\eta}^2} \del{\frac{1}{\sum_u \sqrt{d_u T_u}}}^2, \frac{1}{\tilde{\eta}^2} \min_{u \in [m]} \frac{T_u}{d_u}}.
\] 
Hence, the premises of Lemma~\ref{lem:apriori} is satisfied for $i = i_B$; this gives that with probability $1-\delta/2$,
\begin{equation}
\sum_{t=1}^T{A_t^{i_B} \Delta_t} \leq C \cdot \tilde{\eta} \sum_u{\sqrt{d_u T_u}/\sqrt{\alpha_{i_B}}}, 
\label{eqn:reg-ib}
\end{equation}
and
\begin{equation}
\sum_{t=1}^T{A_t^{i_B} r_t^{i_B}} \leq C \cdot \tilde{\eta}\sqrt{\alpha_{i_B}}\sum_u{\sqrt{d_u T_u}}.
\label{eqn:lc-ib}
\end{equation}
Now from Equation~\eqref{eqn:lc-ib} and the definition of $i_B$, we have
\[ 
\sum_{t=1}^T{A_t^{i_B} r_t^{i_B}} \leq C \cdot \tilde{\eta}\sqrt{\alpha_{i_B}}\sum_u{\sqrt{d_u T_u}} < B'.
\]
Applying Lemma~\ref{lem:trick}, we deduce that for all $t$ in $[T]$, $A_t^{i_B} = 1$. Plugging this back to Equation~\eqref{eqn:reg-ib}, we have
\begin{align*} 
\sum_{t=1}^T \Delta_t = & \sum_{t=1}^T{A_t^{i_B} \Delta_t}  \\
\leq & C \cdot \tilde{\eta} \sum_u{\sqrt{d_u T_u}/\sqrt{\alpha_{i_B}}} \\
\leq & \tilde{O}\del{  \frac{\tilde{\eta}^2 (\sum_u\sqrt{d_u T_u})^2 }{B} }.
\end{align*}
where the second inequality is from the lower bound of $\alpha_{i_B}$ in Equation~\eqref{eqn:ib}.

Combining the above observation with Lemma~\ref{lem:subseq}, along with the union bound, we get that with probability $1 - \delta$, 
\[
R = \sum_{t=1}^T (\hat{y}_t - \inner{\theta^*}{x_t})^2  = \otil{\sum_{t=1}^T{\Delta_t}} = \otil{ \frac{\tilde{\eta}^2 (\sum_u\sqrt{d_u T_u})^2 }{B} }.
\qedhere
\]
\end{proof}

\subsection{Proof of Theorem~\ref{thm:nonlin}}
\label{sec:proof_nonlin}

Define
\begin{align}
    \label{eq:beta_k}
 \beta_k=\beta_k(\Fcal, \delta) \defeq 8 \eta^2 \log{(4\Ncal(\Fcal, 1/T^2, \|\cdot\|_\infty)/\delta)}+2k/T^2 (16+\sqrt{2\eta^2\ln{({16k^2/\delta})}}),
\end{align}
and
\[ 
R_u \defeq \frac{T_u}{T^2}+4\min(d'_u, T_u)+4d'_u\beta_T\ln{T_u} =   
     \tilde{O}\del{ \eta^2 d'_u \log{\Ncal(\Fcal, T^{-2},\| \cdot \|_\infty)}}.
\]

Analogous to Theorem~\ref{thm:upper}, the following theorem provides the query and regret guarantees of of Algorithm~\ref{alg:3}. 
\begin{theorem}
\label{thm:alg3}
Suppose the example sequence $\cbr{x_t}_{t=1}^T$ has the following structure: $[T]$ has an admissible partition $\cbr{I_u: u \in [m]}$, where for each $u$, $|I_u|=T_u$, and the eluder dimension of $\Fcal$ w.r.t. $\cbr{x_t}_{t \in I_u}$ is $d'_u$. Suppose $\alpha \leq \frac{1}{\tilde{\eta}^2} \min_{u \in [m]} \frac{T_u}{R_u}$. With probability $1-\delta$, Algorithm~\ref{alg:3} satisfies:\\
1. Its query complexity $Q= \tilde{O} ( \tilde{\eta} \cdot \sqrt{\alpha} \sum_u \sqrt{R_u T_u}  ) $. \\  
2. Its regret $R= \tilde{O}(\tilde{\eta} \cdot \sum_u {\sqrt{ R_u T_u}})  /\sqrt{\alpha} $.
\end{theorem}

We shall prove Theorem~\ref{thm:nonlin} directly below; the proof of Theorem~\ref{thm:alg3} follows as a corollary, using the same argument in the proof of Theorem~\ref{thm:fixBudget}; we note that the admissibility condition on domain partition $\cbr{I_u}_{u=1}^m$ ensures that $\{A_t^i(r_t^i-p_t^i) \one[t \in I] \}_{t=1}^T$ and $\{A_t^i(r_t^i-p_t^i) \Delta_t \one[t \in I] \}_{t=1}^T$ are still martingale difference sequences in our proof.

\begin{proof}[Proof of Theorem~\ref{thm:nonlin}]
We focus on proving the analogues of Lemma~\ref{lem:subseq} and Lemma~\ref{lem:logdet}; the rest of the proof follows the same argument as the proof of Theorem~\ref{thm:fixBudget} and is therefore omitted.

\begin{lemma}[Analogue of Lemma~\ref{lem:subseq}]
With probability $1-\delta / 2$, $R \le \sum_{t=1}^T{\Delta_t}$.
\end{lemma}
\begin{proof}
Recall that the confidence set at time $t$ is $\Fcal_t = \{ f \in \Fcal: \sum_{i \in \Qcal_t}{(f(x_i)-\hat{f}_t(x_i))^2} \le \beta_{|\Qcal_t|}(\Fcal, \delta) \}$.
By~\citet[Proposition 2]{VR}, we have that with probability $1-\delta/2$, $f^* \in \Fcal_t$, for all $t \in [T]$.

Meanwhile, if $f^* \in \Fcal_t$, for all $t \in [T]$, $(\hat{f}_t(x_t) - f^*(x_t))^2 \leq \sup_{f_1, f_2 \in \Fcal_t} (f_1(x_t) - f_2(x_t))^2 = \Delta_t$.
This implies that the regret is bounded by $R \le \sum_{t=1}^T{\Delta_t}$. 
\end{proof}

\begin{lemma}[Analogue of Lemma~\ref{lem:logdet}]
$\sum_{t \in I_u}{q_t \Delta_t} \le R_u$. 
\end{lemma}
\begin{proof}
Let $k=\abs{I_u \cap \Qcal_T}$ and write $d=d'_u$ as a shorthand. Let $(D_1, \dots, D_k)$ be $\{\Delta_t: t \in I_u \cap \Qcal_T\}$ sorted in non-increasing order. We have
\begin{align*}
    \sum_{t\in I_u \cap \Qcal_T}\Delta_t = \sum_{j=1}^k{D_j} =\sum_{j=1}^k{D_j \one[D_j \le 1/T^4]} + \sum_{j=1}^k{D_j \one[D_j > 1/T^4]}.
\end{align*}
Clearly, $\sum_{j=1}^k{D_j \one[D_j \le 1/T^4]} \le \frac{T_u}{T^2}$.

We know for all $j \in [k]$, $D_j \le 4$. In addition, $D_j > \epsilon^2 \iff \sum_{t\in I_u \cap \Qcal_T}{\one[\Delta_t>\epsilon^2]} \ge j$. By Lemma~\ref{lem:VR3} below, this can only occur if $j < (4\beta_T/\epsilon^2+1)d$. Thus, when $D_j > \epsilon^2$, $j < (4\beta_T/\epsilon^2+1)d$, which implies $\epsilon^2 < \frac{4\beta_T d}{j-d}$. This shows that if $D_j > 1/T^4$, $D_j \le \min\left\{4, \frac{4\beta_T d}{j-d} \right\}$. Therefore $\sum_j{D_j \one[D_j > 1/T^4]} \le 4d+ \sum_{j=d+1}^k{\frac{4\beta_T d}{j-d}}\le 4d + 4 d\beta_T \log{T_u}$.

Consequently,
\begin{align*}
    \sum_{t \in I_u}{q_t \Delta_t} = \sum_{t\in I_u \cap \Qcal_T}\Delta_t \le \min\left\{4T_u, \frac{T_u}{T^2}+4d'_u+4 d'_u\beta_T \log{T_u} \right\} \le R_u. \qquad \qedhere
\end{align*}
\end{proof}

\end{proof}

The following lemma generalizes~\citet[Proposition 3]{VR}, in that it considers a subsequence of examples coming from a subdomain of $\Xcal$. We define $\dim_I^E$ as the eluder dimension of $\Fcal$ with respect to support $\cbr{x_t: t \in I}$. It can be easily seen that $\dim_{I_u}^E \leq \dim_u^E$.

\begin{lemma}
\label{lem:VR3}
Fix $I \subseteq [T]$. If $\{\beta_t \ge 0\}_{t=1}^T$ is a nondecreasing sequence and $\Fcal_t \defeq  \{ f \in \Fcal: \sum_{i \in \Qcal_t}{(f(x_i)-\hat{f}_t(x_i))^2} \le \beta_{|\Qcal_t|}(\Fcal, \delta) \}$, then
\begin{align*}
    \forall  \epsilon>0, \sum_{t \in I \cap \Qcal_T}{\one[\Delta_t>\epsilon^2]} < \left(\frac{4\beta_T}{\epsilon^2}+1\right)\mbox{dim}^E_I(\Fcal, \epsilon).
\end{align*}
\end{lemma}
\begin{proof}
Let $k=\abs{I \cap \Qcal_T}$, $(a_1, \dots, a_{k})=(x_t: t \in I \cap \Qcal_T)$, and $(b_1, \dots, b_{k})=(\Delta_t: t \in I \cap \Qcal_T)$.
First, we show that if $b_j>\epsilon^2$ then $a_j$ is $\epsilon$-dependent on fewer than $4\beta_T/\epsilon^2$ disjoint subsequences of $(a_1, \dots, a_{j-1})$, for $j \le k$, in other words, if there exist $K$ disjoint subsequences of $(a_1,\ldots,a_{j-1})$ such that $a_j$ is $\epsilon$-dependent on all of them, then $K < \frac{4\beta_T}{\epsilon^2}$.

Indeed, suppose $b_j>\epsilon^2$ and $a_j=x_t$, there are $f_1, f_2 \in \Fcal_t$ such that $f_1(a_j)-f_2(a_j)>\epsilon$. By definition, if $a_j$ is $\epsilon$-dependent on a subsequence $(a_{i_1}, \dots, a_{i_p})$ of $(a_1, \dots, a_{j-1})$, then $\sum_{l=1}^p{(f_1(a_{i_l})-f_2(a_{i_l}))^2} >\epsilon^2$. Thus, if $a_j=x_t$ is $\epsilon$-dependent on $K$ subsequences of $(a_1, \dots, a_{j-1})$, then $\sum_{i \in \Qcal_t}{(f_1(x_i)-f_2(x_i))^2} >K \epsilon^2$. By the triangle inequality, $$\sqrt{\sum_{i \in \Qcal_t}{(f_1(x_i)-f_2(x_i))^2}} \le  \sqrt{\sum_{i \in \Qcal_t}{(f_1(x_i)-f^*(x_i))^2}}+\sqrt{\sum_{i \in \Qcal_t}{(f_2(x_i)-f^*(x_i))^2}} \le 2\sqrt{\beta_T}.$$ Thus, $K<4\beta_T/\epsilon^2$.

Next, we show that in any sequence of elements in $I$, $(c_1, \dots, c_\tau)$, there is some $c_j$ that is $\epsilon$-dependent on at least $\tau/d-1$ disjoint subsequences of $(c_1, \dots, c_{j-1})$, where $d \defeq \mbox{dim}^E_I(\Fcal, \epsilon)$. For any integer $K$ satisfying $Kd+1 \le \tau \le Kd+d$, we will construct $K$ disjoint subsequences $C_1, \dots, C_K$. First let $C_i=(c_i)$ for $i \in [K]$. If $c_{K+1}$ is $\epsilon$-dependent on $C_1, \dots, C_K$, our claim is established. Otherwise, select a $C_i$ such that $c_{K+1}$ is $\epsilon$-independent and append $c_{K+1}$ to $C_i$. Repeat for all $j>K+1$ until $c_j$ is $\epsilon$-dependent on each subsequence or $j=\tau$. In the latter case $\sum{\abs{C_i}} \ge Kd$, and $\abs{C_i}=d$. In this case, $c_\tau$ must be $\epsilon$-dependent on each subsequence, by the definition of $\dim^E_I$.

Now take $(c_1, \dots, c_\tau)$ to be the subsequence $(a_{t_1}, \dots, a_{t_\tau})$ of $(a_1, \dots, a_k)$ consisting of elements $a_j$ for which $b_j>\epsilon^2$. We proved that each   $a_{t_j}$ is $\epsilon$-dependent on fewer than $4\beta_T/\epsilon^2$ disjoint subsequences of $(a_1, \dots, a_{t_j-1})$. Thus, each $c_j$ is $\epsilon$-dependent on fewer than $4\beta_T/\epsilon^2$ disjoint subsequences of $(c_1, \dots, c_{j-1})$.\footnote{To see this, observe that if $c$ is $\epsilon$-dependent on a sequence $S$, then $c$ must also be $\epsilon$-dependent on any supersequence of $S$.} Combining this with the fact that there is some $c_j$ that is $\epsilon$-dependent on at least $\tau/d-1$ disjoint subsequences of $(c_1, \dots, c_{j-1})$, we have $\tau/d-1 < 4\beta_T/\epsilon^2$. Thus, $\tau < (4\beta_T/\epsilon^2+1)d$.
\end{proof}

\begin{algorithm}
\caption{Fixed-budget QuFUR for general function class \label{alg:4}}
\begin{algorithmic}[1]

\Require Hypotheses set $\Fcal$, time horizon $T$, label budget $B$, parameter $\delta$, noise level $\eta$.
\State Labeled dataset $\Qcal\leftarrow \emptyset$.
\State $k \leftarrow 3 \lceil \log_2{T} \rceil$.
\For {$i=0$ to $k$}
    \State Parameter $\alpha_i \leftarrow 2^i /T^2$.
\EndFor
\For {$t=1$ to $T$}
    \State Predict $\hat{f}_t \leftarrow \argmin_{f \in \Fcal}{\sum_{i \in \Qcal}{(f(x_i)-y_i)^2}}$.
    \State Confidence set $\Fcal_t \leftarrow \{ f \in \Fcal: \sum_{i \in \Qcal}{(f(x_i)-\hat{f}(x_i))^2} \le \beta_{|\Qcal|}(\Fcal, \delta) \}$, \\ where $\beta_{k} \defeq 8 \eta^2 \log{(4\Ncal(\Fcal, 1/T^2, \|\cdot\|_\infty)/\delta)}+2k/T^2 (16+\sqrt{2\eta^2\ln{({16k^2/\delta})}})$.
    \State Uncertainty estimate $\Delta_t = \sup_{f_1, f_2 \in \Fcal_t}{\abs{f_1(x_t)-f_2(x_t)}^2}$.
    \For {$i=0$ to $k$}
        \If {$\sum_{j=1}^{t-1}{q_j^i} < \lfloor B/k \rfloor$}
            \State With probability $\min{\{1, \alpha_i \Delta_t\}}$, set $q_t^i=1$.
        \EndIf
    \EndFor
    \If {$\sum_i{q_t^i}>0$}
        \State Query $y_t$. $\Qcal \leftarrow \Qcal \bigcup \{t\}$.
    \EndIf
\EndFor

\end{algorithmic}
\end{algorithm}

\subsection{Analysis of uniform query strategy for online active linear regression with oblivious adversary}
\label{sec:proof_uniform}
\begin{theorem}
With probability $1-\delta$, the uniformly querying strategy with probability $\mu$ achieves $\EE[R] =\otil{ \frac{\tilde{\eta}^2 d} \mu}$ and $\EE[Q]= \mu T$.
\end{theorem}

\begin{proof}[Proof sketch]

As $Q = \sum_{t=1}^T q_t$ is a sum of $T$ iid Bernoulli random variables with means $\mu$, $\EE[Q] = \mu T$.

We now bound the regret of the algorithm. We still define $\Delta_t = \tilde{\eta}^2 \min\{1, \|x_t\|_{M_t^{-1}}^2\}$.

Using Lemma~\ref{lem:logdet} with $\{a_i\}_{i=1}^k= \cbr{x_t}_{t=1}^T$, and $S = \Qcal_T$, $\sum_t{q_t \Delta_t} = \tilde{O}(\tilde{\eta}^2 d)$.
Let $Z_t=q_t \Delta_t$. 
We have $Z_t \le \Delta_t \le \tilde{\eta}^2$, $\EE_{t-1}Z_t = \mu \Delta_t$, and $\EE_{t-1}Z_t^2 \le  \tilde{\eta}^2 \mu \Delta_t$. Applying Freedman's inequality, with probability $1-\delta/2$, 
\begin{align*}
    \sum_{t=1}^T{\mu \Delta_t} - \sum_{t=1}^T{q_t \Delta_t} =O\left(\tilde{\eta} \sqrt{\sum_{t=1}^T{\mu \Delta_t \ln{(\ln{T}/\delta})} }+\tilde{\eta}^2 \ln{(\ln{T}/\delta)})\right).
\end{align*}
The above inequality implies that $\sum_{t=1}^T{\Delta_t} = \otil{ \frac{\tilde{\eta}^2 d} \mu}$.
Now, applying Lemma~\ref{lem:subseq} and take the union bound, we have that with probability $1-\delta$, 
\[ R = \otil{\sum_{t=1}^T{\Delta_t}} = \otil{ \frac{\tilde{\eta}^2 d} \mu}. \]
Use the basic relationship between the expectation and  tail probability $\EE[R] = \int_0^\infty \PP(R \geq a) d a$, we conclude that $\EE[R] = \otil{ \frac{\tilde{\eta}^2 d} \mu}$.
\end{proof}

\section{Proofs for lower bounds}
\subsection{Proof of Theorem~\ref{thm:lower}}
\label{sec:proof_lower}

\begin{reptheorem}{thm:lower}

For any $\eta \geq 1$, any set of positive integers $\cbr{(d_u, T_u)}_{u=1}^m$ and integer $B$ that satisfy
\begin{equation*} 
d_u \leq T_u, \forall u \in [m], \quad \sum_{u=1}^m d_u \leq d, \quad B \geq \sum_{u=1}^m d_u, 
\end{equation*}
there exists an oblivious adversary such that:\\
1. It uses a ground truth linear predictor $\theta^\star \in \RR^d$ such that $\| \theta^* \|_2 \leq \sqrt{d}$, and $\abs{\inner{\theta^*}{x_t}} \leq 1$; in addition, the noises $\cbr{\xi_t}_{t=1}^T$ are sub-Gaussian with variance proxy $\eta^2$.\\
2. It shows example sequence $\cbr{x_t}_{t=1}^T$ such that $[T]$ can be partitioned into $m$ disjoint nonempty subsets $\cbr{I_u}_{u=1}^m$, where for each $u$, $|I_u|=T_u$, and  $\cbr{x_t}_{t \in I_u}$ lie in a subspace of dimension $d_u$.\\
3. Any online active learning algorithm $\cal{A}$ with label budget $B$ has regret $\Omega((\sum_{u=1}^m \sqrt{d_u T_u})^2/B)$.
\end{reptheorem}

\begin{proof}
Our proof is inspired by~\citet[Theorem 2]{vovk2001competitive}.
For $u \in [m]$ and $i \in [d_u]$, define $c_{u,i} = e_{\sum_{v=1}^{u-1} d_v + i}$, where $e_j$ denotes the $j$-th standard basis of $\RR^d$. It can be easily seen that all $c_{u,i}$'s are orthonormal. In addition, for a vector $\theta \in \RR^d$, denote by $\theta_{u,i} = \theta_{\sum_{v=1}^{u-1} d_v + i}$.

For task $u$, we construct domain $\Xcal_u = \cspan(c_{u,i}: i \in [d_u])$. The sequence of examples shown by the adversary is the following: it is divided to $m$ blocks, where the $u$-th block occupies a time interval $I_u = [\sum_{v=1}^{u-1} T_v +1, \sum_{v=1}^{u} T_v ]$; Each block is further divided to $d_u$ subblocks, where for $i \in [d_u - 1]$, subblock $(u,i)$ spans time interval $I_{u,i} = [\sum_{v=1}^{u-1} T_v + (i-1) \lfloor T_u / d_u \rfloor +1, \sum_{v=1}^{u-1} T_v + i \lfloor T_u / d_u \rfloor]$, and subblock $(u, d_u)$ spans time interval $I_{u,d_u} = [\sum_{v=1}^{u-1} T_v + (d_u-1) \lfloor T_u / d_u \rfloor +1, \sum_{v=1}^{u-1} T_v + T_u]$. 
At block $u$, examples from domain $\Xcal_u$ are shown; furthermore, for every $t$ in $I_{u,i}$, i.e. in the $(u,i)$-th subblock, example $c_{u,i}$ is repeatedly shown to the learner.
Observe that $(u,i)$-th subblock contains at least $\lfloor \frac{T_u}{d_u} \rfloor \geq \frac{T_u}{2 d_u}$ examples, as $T_u \geq d_u$.

We first choose $\theta^*$ from distribution $D_\theta$, such that for every coordinate $j \in [d]$, $\theta_i^* \sim \Beta(1,1)$, which is also the uniform distribution over $[0,1]$.
Given $\theta^*$, the adversary reveals labels using the following mechanism: given $x_t$, it draws $y_t \sim \Bernoulli(\inner{\theta^*}{x_t})$ independently and optionally reveals it to the learner upon learner's query. Specifically, given $\theta^*$, if $t \in I_{u,i}$,  $y_t \sim \Bernoulli(\theta^*_{u,i})$.
By Hoeffding's Lemma, $\xi_t = y_t - \theta^*_{u,i}$ is zero mean subgaussian with variance proxy $\frac{1}{4} \leq \eta^2$.

Denote by $N_{u,i}(t) = \sum_{s \in I_{u,i}: s \leq t} q_s$ the number of label queries of the learner in domain $(u,i)$ up to time $t$. Because the learner satisfies a budget constraint of $B$ under all environments, we have
\[ \EE \sbr{ \sum_{u=1}^m \sum_{i=1}^{d_u} N_{u,i}(T) \mid \theta^* } \leq B.  \]
Adding $2\sum_{u=1}^m d_u$ on both sides and by linearity of expectation, we get
\begin{equation} 
\sum_{u=1}^m \sum_{i=1}^{d_u} \EE \sbr{  (N_{u,i}(T) + 2) \mid \theta^* } \leq B + 2\sum_{u=1}^{m} d_u \leq 3B. 
\label{eqn:qc-ub}
\end{equation}

On the other hand, we observe that the expected regret of the algorithm can be written as follows:
\[ 
\EE\sbr{R} = \EE \sbr{ \sum_{u=1}^m \sum_{i=1}^{d_u} \sum_{t \in I_{u,i}} (\hat{y}_t - \theta_{u,i}^*)^2},
\]
where the expectation is with respect to both the choice of $\theta^*$ and the random choices of $\Acal$.

We define a filtration $\cbr{\Fcal_t}_{t=1}^T$, where $\Fcal_t$ is the $\sigma$-algebra generated by $\cbr{(x_s, q_s, y_s q_s)}_{s=1}^t$, which encodes the informative available to the {\em learner} up to time step $t$.\footnote{This notion should be distinguished from the history notion $H_t$ defined before, in that it does not include the labels not queried by the learner up to time step $t$. For $s$ in $[t]$, we use $y_s q_s$ to indicate the labeled data information acquired at time step $s$; if $q_s = 1$, $y_sq_s = y_s$, encoding the fact that the learner has access to label $y_s$; otherwise $q_s = 0$, $y_sq_s$ is always $0$, meaning that the learner does not have label $y_s$ available.} We note that $\hat{y}_t$ is $\Fcal_{t-1}$-measurable. 
Denote by $N_{u,i}^+(t) = \sum_{s \in I_{u,i}: s \leq t} q_s \cdot \one\del{y_s = 1 }$, which is the number of $1$ labels seen on example $c_{u,i}$ by the learner up to round $t-1$.
Observe that both $N_{u,i}^+(t-1)$ and $N_{u,i}(t-1)$ are $\Fcal_{t-1}$-measurable. 

Observe that conditioned on the interaction logs $(x_s, q_s, y_s q_s)_{s=1}^{t-1}$, the posterior distribution of
$\theta_{u,i}^*$ is $\Beta(1+N_{u,i}^+(t-1), 1+N_{u,i}(t-1)-N_{u,i}^+(t-1))$. 
Therefore, define random variable $\hat{y}_t^* = \EE\sbr{ \theta_{u,i}^* \mid \Fcal_{t-1} } =  \frac{1+N_{u,i}^+}{2+N_{u,i}}$, we have by bias-variance decomposition,
\begin{align*} 
\EE \sbr{(\hat{y}_t - \theta_{u,i})^2 \mid \Fcal_{t-1}} &= \EE \sbr{ (\hat{y}_t^* - \theta_{u,i}^*)^2 \mid \Fcal_{t-1}} + (\hat{y}_t - \hat{y}_t^*)^2 \\
&\geq \EE \sbr{ (\hat{y}_t^* - \theta_{u,i}^*)^2 \mid \Fcal_{t-1}}
\end{align*}

Summing over all time steps, we have
\[
\EE\sbr{R} \geq \EE \sbr{ \sum_{u=1}^m \sum_{i=1}^{d_u} \sum_{t \in I_{u,i}} (\hat{y}_t^* - \theta_{u,i}^*)^2}.
\]

On the other hand, from Lemma~\ref{lem:reg-lb-qc}, we have for all $t \in I_{u,i}$, 
\[ \EE \sbr{ (\hat{y}_t - \theta_{u,i}^* )^2 \mid N_{u,i}(T), \theta^* } \geq \frac{f(\theta_{u,i}^*)}{2(N_{u,i}(T) + 2)}, \]
where $f(\gamma) = \min( \gamma \cdot (1-\gamma), (2\gamma - 1)^2)$.

By the tower property of conditional expectation and conditional Jensen's inequality, we have
\[ \EE \sbr{ (\hat{y}_t - \theta_{u,i})^2 \mid \theta^* } \geq \EE \sbr{ \frac{f(\theta_{u,i}^*)}{N_{u,i}(T) + 2} \mid \theta^* } \geq \frac{f(\theta_{u,i}^*)}{2(\EE\sbr{N_{u,i}(T) \mid \theta^*} + 2)}. \]
Summing over all $t$ in $I_{u,i}$, and then summing over all subblocks $(u,i): u \in [m], i \in [d_u]$, and using the aforementioned fact that the $(u,i)$ subblock has at least $\frac{T_u}{2d_u}$ examples, we have
\begin{align}
\EE \sbr{ R \mid \theta^* } & =  \sum_{u=1}^m \sum_{i=1}^{d_u} \sum_{t \in I_{u,i}} \EE \sbr{ (\hat{y}_t - \theta_{u,i})^2 \mid \theta^*} \nonumber \\
& \geq \sum_{u=1}^m \sum_{i=1}^{d_u} \frac{T_u / d_u \cdot f(\theta_{u,i}^*)}{4(\EE\sbr{N_{u,i}(T) \mid \theta^*} + 2)}.
\end{align}
Combining the above inequality with Equation~\eqref{eqn:qc-ub}, we have:
\begin{align*}
3B \cdot \EE \sbr{ R \mid \theta^*} & \geq \del{\sum_{u=1}^m \sum_{i=1}^{d_u} \frac{T_u / d_u \cdot f(\theta_{u,i}^*)}{4(\EE\sbr{N_{u,i}(T) \mid \theta^*} + 2)}} \cdot \del{\sum_{u=1}^m \sum_{i=1}^{d_u} \EE \sbr{  (N_{u,i}(T) \mid \theta^*} + 2) } \\
& \geq \frac{1}{4}\del{ \sum_{u=1}^m \sum_{i=1}^{d_u} \del{\sqrt{T_u / d_u} \cdot \sqrt{f(\theta_{u,i}^*)}} }^2 .
\end{align*}
where the second inequality is from Cauchy-Schwarz. Now taking expectation over $\theta$, using Jensen's inequality and Lemma~\ref{lem:e-f-theta} that $\EE \sqrt{f(\theta_{u,i}^*)} \geq \frac{1}{25}$, and some algebra yields
\[
3B \cdot \EE \sbr{R} \geq \frac{1}{2}\del{ \sum_{u=1}^m \sum_{i=1}^{d_u} \del{\sqrt{T_u / d_u} \cdot \EE\sbr{\sqrt{f(\theta_{u,i}^*)}}} }^2 \geq \frac{1}{2500}  \del{ \sum_{u=1}^m \sqrt{d_u T_u}  }^2.
\]
 
In conclusion, we have
\[ \EE \sbr{R} \geq \frac{\del{ \sum_{u=1}^m \sum_{i=1}^{d_u} \sqrt{T_u / d_u} }^2}{7500 \cdot B}. \]
As the above expectation is over $\theta^*$ chosen randomly from $D_\theta$, there must exists an $\theta^*$ from $\supp(D_\theta) = [0,1]^d$ such that 
\[ \EE \sbr{ R \mid \theta^*} \geq \frac{\del{ \sum_{u=1}^m \sum_{i=1}^{d_u} \sqrt{T_u / d_u} }^2}{7500 \cdot B} \]
holds. This $\theta^*$ has $\ell_2$ norm at most $\sqrt{\sum_{j=1}^d (\theta_j^*)^2} \leq \sqrt{d}$.
\end{proof}

\begin{lemma}
If $t$ is in $I_{u,i}$, then 
\[ \EE \sbr{ (\hat{y}_t^* - \theta_{u,i}^*)^2 \mid N_{u,i}(T), \theta^* } \geq \frac{f(\theta_{u,i}^*)}{2(N_{u,i}(T) + 2)}, \]
where $f(\gamma) = \min\del{ \gamma(1-\gamma), (2\gamma - 1)^2 }$.
\label{lem:reg-lb-qc}
\end{lemma}
\begin{proof}
We condition on $N_{u,i}(T) = m$, and a value of $\theta^*$.
Recall that $\hat{y}_t^* = \frac{1+N_{u,i}^+}{2+N_{u,i}} = \frac{1+N_{u,i}^+}{2+m}$, where $N_{u,i}^+$ can be seen as drawn from the binomial distribution $\Bin(m, \theta_{u,i}^*)$. 

\begin{align*}
&\EE \sbr{ (\hat{y}_t^* - \theta_{u,i}^*)^2 \mid N_{u,i}(T) = m, \theta^* }  \\
=&\EE \sbr{ \left(\frac{1+N_{u,i}^+}{2+m} - \theta_{u,i}^*\right)^2 \mid N_{u,i}(T) = m, \theta^* } \\
= &\frac{m \theta_{u,i}^* (1-\theta_{u,i}^*) }{(m+2)^2} + \frac{ (2\theta_{u,i}^* - 1)^2 }{(m+2)^2} \\
\geq& \frac{m+1}{(m+2)^2} f(\theta_{u,i}^*) 
\geq \frac{f(\theta_{u,i}^*)}{2(m+2)}. \qedhere
\end{align*}

\end{proof}

\begin{lemma}
Suppose $Z \sim \Beta(1,1)$. Then 
$\EE \sbr{\sqrt{f(Z)}} \geq \frac{1}{25}$.
\label{lem:e-f-theta}
\end{lemma}
\begin{proof}
We observe that 
\[ \EE \sbr{\sqrt{f(Z)}} = \int_{[0,1]} \sqrt{f(z)} dz \geq \int_{[\frac15, \frac25]} \sqrt{f(z)} dz, \]
Now, for all $z \in [\frac15, \frac25]$, $\sqrt{f(z)} \geq \sqrt{\frac{1}{25}} = \frac15$, which implies that the above integral is at least $\frac1{25}$.
\end{proof}

\subsection{Lower bound for unstructured domains}
\label{sec:lower-unstructured}
 
We have the following lower bound in the case when there is no domain structure. 
\begin{theorem}
\label{thm:lower-unstructured}
For any set of positive integers $d, T, B$ such that $d \leq T$ and $d \leq B$, there exists an oblivious adversary such that:
\begin{enumerate}
\item it uses a ground truth linear predictor $\theta^\star \in \RR^d$ such that $\| \theta^* \|_2 \leq \sqrt{d}$, and $\abs{\inner{\theta^*}{x_t}} \leq 1$.
\item any online active learning algorithm $\cal{A}$ with label budget $B$ has regret at least $\Omega\del{\frac{d T}{B}}$.
\end{enumerate}
\end{theorem}
\begin{proof}
This is an immediate consequence of Theorem~\ref{thm:lower}, by setting $m = 1$, $d_1 = d$, $T_1 = T$, and the label budget equal to $B$.
\end{proof}

\section{The c-cost model for online active learning}
\label{sec:cost-model}

We consider the following variant of our learning model, which models settings where the cost ratio between a unit of square loss regret and a label query is $c$ to 1. In this setting, the interaction protocol between the learner and the environment remains the same, with the goal of the learner modified to minimizing the total cost, formally $W = c R + Q$. 
We call the above model the {\em $c$-cost model}.
We will show that Algorithm~\ref{alg:1} achieves optimal cost up to constant factors, for a wide range of values of $\eta$ and $c$. 

\begin{theorem}
For any $\eta \geq 1$, set of positive integers $\cbr{(d_u, T_u)}_{u=1}^m$ such that $d_u \leq T_u, \forall u \in [m]$, $\sum_{u=1}^m d_u \leq d$, cost ratio $c \geq \max_u \frac{d_u}{T_u}$, there exists an oblivious adversary such that:
\begin{enumerate}
\item it uses a ground truth linear predictor $\theta^\star \in \RR^d$ such that $\| \theta^* \|_2 \leq \sqrt{d}$, and $\abs{\inner{\theta^*}{x_t}} \leq 1$; in addition, the subgaussian variance proxy of noise is $\eta^2$. 
\item it shows example sequence $\cbr{x_t}_{t=1}^T$ such that $[T]$ can be partitioned into $m$ disjoint nonempty subsets $\cbr{I_u}_{u=1}^m$, where for each $u$, $|I_u|=T_u$, and  $\cbr{x_t}_{t \in I_u}$ lie in a subspace of dimension $d_u$.
\item any online active learning algorithm $\cal{A}$ has total cost $\Omega\del{\sqrt{c} \cdot (\sum_{u=1}^m \sqrt{d_u T_u})}$.
\end{enumerate}
\label{thm:lower-cost}
\end{theorem}

\begin{proof}
Consider any algorithm $\Acal$. Same as in the proof of Theorem~\ref{thm:lower}, we will choose $\theta^*$
randomly where each of its coordinates is drawn independently from the $\Beta(1,1)$ distribution, and show the exact same sequence of instances $\cbr{x_t}_{t=1}^T$ and reveals the labels the same say as in that proof. It can be seen that the $\eta_t$'s are subgaussian with variance proxy $1$, which is also subgaussian with variance proxy $\eta^2$.

As $\Acal$ can behave differently under different environments, we define $\EE\sbr{Q \mid \theta^*}$ as $\Acal$'s query complexity conditioned on the adversary choosing ground truth linear predictor $\theta^*$. 

We conduct a case analysis on the random variable $\EE\sbr{Q \mid \theta^*}$:
\begin{enumerate}
\item If there exists some $\theta^* \in [0,1]^d$, $\EE\sbr{Q \mid \theta^*} \geq \sqrt{c} \del{\sum_{u=1}^m \sqrt{d_u T_u}}$, then we are done: under the environment where the ground truth linear predictor is $\theta^*$, the total cost of $\Acal$, $\EE\sbr{W \mid \theta^*}$, is clearly at least $\EE\sbr{Q \mid \theta^*} \geq \Omega\del{ \sqrt{c} \del{\sum_{u=1}^m \sqrt{d_u T_u}} }$.

\item If for every $\theta^* \in [0,1]^d$, $\EE\sbr{Q \mid \theta^*} \leq  \sqrt{c} \del{\sum_{u=1}^m \sqrt{d_u T_u}}$, $\Acal$ can be viewed as an algorithm with label budget $B = \sqrt{c} \del{\sum_{u=1}^m \sqrt{d_u T_u}}$. 
By the premise that $c \geq \max_u \frac{d_u}{T_u}$, we get that $B \geq \sum_{u=1}^m \sqrt{d_u T_u} \cdot \sqrt{\frac{d_u}{T_u}} = \sum_{u=1}^m d_u$.
Therefore, from the proof of Theorem~\ref{thm:lower}, we get that there exists a $\theta^*$ in $[0,1]^d$, such that 
\[
\EE\sbr{R \mid \theta^*} 
\geq
\frac{(\sum_u \sqrt{d_u T_u})^2}{B}
\geq
\Omega\del{ \frac1{\sqrt{c}} \del{ \sum_u \sqrt{d_u T_u} } },
\]
which implies that the total cost of $\Acal$, under the environment where the ground truth linear predictor is $\theta^*$, $\EE\sbr{W \mid \theta^*}$, is at least $c \cdot \EE\sbr{R \mid \theta^*} \geq \Omega\del{ \sqrt{c} \del{ \sum_u \sqrt{d_u T_u} }}$.
\end{enumerate}
In summary, in both cases, there is an oblivious adversary that uses $\theta^*$ in $[0,1]^d$, under which $\Acal$ has a expected cost of $\Omega\del{ \sqrt{c} \del{ \sum_u \sqrt{d_u T_u} }}$.
\end{proof}

In the theorem below, we discuss the optimality of Algorithm~\ref{alg:1} in the $c$-cost model for a range of problem parameters.

\begin{theorem}
Suppose $\eta \in [1, O(1)]$; in addition, consider a set of $\cbr{(T_u, d_u)}_{u=1}^m$, such that $\min_u T_u / d_u \geq \eta$. Fix $c \in [\max_u \frac{d_u}{T_u}, \frac{1}{\eta^2} \min_u \frac{T_u}{d_u} ]$. We have
\begin{enumerate}
\item Under all environments with domain dimension and duration $\cbr{(T_u, d_u)}_{u=1}^m$, such that $\| \theta^* \| \leq C$ and $\max_{t \in [T]}\abs{\inner{\theta^*}{x_t}} \leq 1$,
$\qufur(c)$ (with the knowledge of norm bound $C$) has the guarantee that 
\[ 
W \leq \tilde{O}\del{\sqrt{c} \cdot \sum_{u} \sqrt{T_u d_u}},
\]

\item For any algorithm, there exists an environment with domain dimension and duration $\cbr{(T_u, d_u)}_{u=1}^m$ such that $\| \theta^* \| \leq \sqrt{d}$ and $\max_{t \in [T]}\abs{\inner{\theta^*}{x_t}} \leq 1$, under which the algorithm must have the following cost lower bound:
\[
W \geq \Omega\del{\sqrt{c} \cdot \sum_{u} \sqrt{T_u d_u}},
\]
\end{enumerate}
\end{theorem}

\begin{proof}
We show the two items respectively:
\begin{enumerate}
    \item As $c \leq \tilde{\eta}^2 \min_u \frac{T_u}{d_u}$,  
    and $c \geq \max_u \frac{d_u}{T_u} \geq \frac{1}{\tilde{\eta}^2} (\frac{1}{(\sum_u \sqrt{d_u T_u})^2})$,
    applying Theorem~\ref{thm:upper}, we have that
     $\qufur(c)$ achieves the following regret and query complexity guarantees:
     \[ 
     Q \leq \order{ \tilde{\eta} \sqrt{c}  \sum_{u} \sqrt{T_u d_u} }, \quad
     R \leq \order{ \tilde{\eta} \sum_{u} \sqrt{T_u d_u} / \sqrt{c} }.
     \]
     This implies that 
     \[ 
        W
        = 
        c Q + R 
        \leq 
        \order{ \tilde{\eta} \sum_{u} \sqrt{T_u d_u} \cdot \sqrt{c} } 
        =  
        \order{\sqrt{c} \cdot \sum_{u} \sqrt{T_u d_u}}. 
    \]
    \item By the condition that $c \geq \max_u \frac{d_u}{T_u}$, applying Theorem~\ref{thm:lower-cost}, we get the item.
    \qedhere
\end{enumerate}
\end{proof}

\section{The regret definition}
\label{sec:reg-def}
Recall that in the main text, we define the regret of an algorithm as $R = \sum_{t=1}^T (\hat{y}_t - f^*(x_t))^2$. 
This is different from the usual definition of regret in online learning, which measures the difference between the 
loss of the learner and that of the predictor $f^*$:
$\Reg = \sum_{t=1}^T (\hat{y}_t - y_t)^2 - \sum_{t=1}^T (f^*(x_t) - y_t)^2$.

We show a standard result in this section that the expectation of these two notions coincide.

\begin{theorem}\label{thm:reg-equiv}
$\EE[R] = \EE[\Reg]$.
\end{theorem}
\begin{proof}
Denote by $\Fcal_{t-1}$ be the $\sigma$-algebra generated by all
observations up to time $t-1$, and $x_t$. As a shorthand, denote by $\EE_{t-1}[\cdot] = \EE[\cdot \mid \Fcal_{t-1}]$.

Let $Z_t = (\hat{y}_t - y_t)^2 - (f^*(x_t) - y_t)^2$; we have 
\begin{align*}
    \EE_{t-1} Z_t
    & = \EE_{t-1} \sbr{ (\hat{y}_t - f^*(x_t) + f^*(x_t) - y_t)^2 - (f^*(x_t) - y_t)^2 } \\
    & = \EE_{t-1} \sbr{ (f^*(x_t) - \hat{y}_t)^2 + 2 (\hat{y}_t - f^*(x_t)) (f^*(x_t) - y_t) } \\
    & = (f^*(x_t) - \hat{y}_t)^2 
\end{align*}
where the last inequality uses the fact that $\EE_{t-1}(f^*(x_t)-y_t) = 0$ and $\hat{y}_t - f^*(x_t)$ is $\Fcal_{t-1}$-measurable. Consequently, $\EE Z_t = \EE (f^*(x_t) - \hat{y}_t)^2$. The theorem is concluded by summing over all time steps $t$ from $1$ to $T$. 
\end{proof}

\section{Online to batch conversion}

In this section we show that by an standard application of online to batch conversion~\cite{cesa2004generalization} on \qufur, we obtain new results on active linear regression under the batch learning setting. 

First we recall a standard result on online to batch conversion; for completeness we provide its proof here.
\begin{theorem}
\label{thm:otb}
Suppose online active learning algorithm $\Acal$ sequentially receives a set of iid examples $(x_t, y_t)_{t=1}^T$ drawn from $D$, and at every time step $t$, it outputs predictor $\hat{f}_t: \Xcal \to \Ycal$. In addition, suppose $\ell: \Ycal \times \Ycal \to \RR$ is a loss function. Define regret $\Reg = \sum_{t=1}^T \ell(\hat{f}_t(x_t), y_t) - \sum_{t=1}^T \ell(f^*(x_t), y_t)$, and define $\ell_D(f) = \EE_{(x,y) \sim D} \ell(f(x), y)$.
If
$
   \EE\sbr{\Reg} \leq R_0,
$
then, 
\[
  \EE \sbr{ \EE_{f \sim \unif(\hat{f}_1, \ldots, \hat{f}_T)} \ell_D(f)} - \ell_D(f^*) \leq \frac{R_0}{T}.
\]
\end{theorem}

\begin{proof}
As $\Reg = \sum_{t=1}^T \ell(\hat{f}_t(x_t), y_t) - \sum_{t=1}^T \ell(f^*(x_t), y_t)$, We have
\begin{align*}
R_0 \geq \EE\sbr{\Reg} 
& = \sum_{t=1}^T \EE\sbr{ \ell_D(\hat{f}_t) } - \EE \sbr{\sum_{t=1}^T \ell(f^*(x_t), y_t)} \\
& =
T \cdot \del{ \frac1T\sum_{t=1}^T \EE\sbr{ \ell_D(\hat{f}_t) } 
-
\EE_{(x,y \sim D} \ell(f^*(x), y).
}
\end{align*}

The theorem is proved by dividing both sides by $T$ and recognizing that 
\[ 
\frac1T \sum_{t=1}^T \EE\sbr{ \ell_D(\hat{f}_t) } = \EE_{f \sim \unif(\hat{f}_1, \ldots, \hat{f}_T)} \ell_D(f). \qedhere 
\]
\end{proof}
Combining Theorem~\ref{thm:otb} with Theorem~\ref{thm:fixBudget}, we have the following adaptive excess loss guarantee of Fixed-Budget QuFUR (Algorithm~\ref{alg:2}) when run on iid data with hidden domain structure.

\begin{theorem}
Suppose the unlabeled data distribution $D_X$ is a mixture distribution: $D_X = \sum_{u=1}^m p_u D_u$, where $D_u$ is a distribution supported on a subspace of $\RR^d$ of dimension $d_u$ and is a subset of $\cbr{x: \| x \|_2 \leq 1, \abs{\inner{\theta^*}{x}} \leq 1}$. The conditional distribution of $y$ given $x$ is $y = \inner{\theta^*}{x} + \xi$ where $\xi$ is a subgaussian with variance proxy $\eta^2$.
In addition, suppose we are given integer $B$, $T_0$ such that $T_0 \geq \Omega\del{\max\del{\frac{B}{\sum_u \sqrt{d_u p_u} \cdot \min_u \sqrt{\frac{p_u}{d_u}}}, \frac{\ln m}{\min_u p_u}}}$.
If Algorithm~\ref{alg:2} is given dimension $d$, time horizon $T \geq T_0$, label budget $B$, norm bound $C$, noise level $\eta$ as input, then:\\
1. It uses $T$ unlabeled examples. \\ 
2. Its query complexity $Q$ is at most $B$. \\
3. Denote by $\ell(\hat{y}, y) = (\hat{y} - y)^2$ the square loss.
We have,
\[ 
\EE \sbr{ \EE_{f \sim \unif(\hat{f}_1, \ldots, \hat{f}_T)} \ell_D(f)} - \ell_D(f^*) 
\leq 
\order{ \frac{\tilde{\eta}^2 (\sum_u \sqrt{d_u p_u})^2}{B} }.
\]
\label{cor:otb-linear}
\end{theorem}

\begin{proof}[Proof sketch]
From Theorem~\ref{thm:otb} it suffices to show that
\[ \EE\sbr{\Reg} \leq  \order{ \frac{\tilde{\eta}^2 T \cdot  (\sum_u \sqrt{d_u p_u})^2}{B} }. \]
By Theorem~\ref{thm:reg-equiv}, $\EE\sbr{\Reg} = \EE\sbr{R}$, it therefore suffices to show that
\[
\EE\sbr{R} \leq \order{ \frac{\tilde{\eta}^2 T \cdot  (\sum_u \sqrt{d_u p_u})^2}{B} }.
\]

We first show a high probability upper bound of $R$. Given a sequence of unlabeled examples $\cbr{x_t}_{t=1}^T$, we denote by $S_u$ the subset of examples drawn from component $D_u$, and denote by $T_u$ the size of $S_u$. From the assumption of $D_u$, we know that $S_u$ all lies in a subspace of dimension $d_u$.

Define event $E$ as follows:
\[
E = \cbr{ \forall u \in [m] \centerdot T_u \in \intcc{\frac{T p_u}{2}, 2 T p_u} }.
\]
From the assumption that $T \geq T_0 \geq \Omega(\frac{\ln m}{\min_u p_u})$, we have that by Chernoff bound and union bound, $\PP(E) \geq 1-\frac{1}{T^2}$.

Conditioned on event $E$ happening, we have that by the assumption that $T \geq T_0 \geq \frac{B}{\sum_u \sqrt{d_u p_u} \cdot \min_u \sqrt{\frac{p_u}{d_u}}}$,
\[ 
   B 
   \leq 
   \otil{T \cdot \sum_u \sqrt{d_u p_u} \min_u \sqrt{\frac{p_u}{d_u}} } 
   \leq 
   \otil{\sum_u \sqrt{d_u T_u} \min_u \sqrt{\frac{T_u}{d_u}} }. 
\]
Therefore, applying Theorem~\ref{thm:fixBudget}, we have that conditioned on event $E$ happening, with probability $1 - \frac1{T^2}$ over the draw of $\cbr{y_t}_{t=1}^T$,
\[ 
   R 
   \leq   
   \order{ \frac{\tilde{\eta}^2 \cdot  (\sum_u \sqrt{d_u T_u})^2}{B} } 
   \leq 
   \order{ \frac{\tilde{\eta}^2 T \cdot  (\sum_u \sqrt{d_u p_u})^2}{B} } .
\]
Combining the above two equations and using union bound, we conclude that with probability $1-\frac2{T^2}$,
\[
   R 
   \leq 
   \order{ \frac{\tilde{\eta}^2 T \cdot  (\sum_u \sqrt{d_u p_u})^2}{B} }.
\]

Observe that with probability $1$, $\hat{y}_t \in [-1,1]$ and $\inner{\theta^*}{x_t} \in [-1,1]$. Therefore, $R = \sum_{t=1}^T (\hat{y}_t - \inner{\theta^*}{x_t})^2  \in [0, 4T]$. Hence,
\[ 
  \EE[R] 
  \leq 
  \del{1-\frac2{T^2}} \cdot \order{ \frac{\tilde{\eta}^2 T \cdot  (\sum_u \sqrt{d_u p_u})^2}{B} } + \frac{2}{T^2} \cdot 4T
  =
  \order{ \frac{\tilde{\eta}^2 T \cdot  (\sum_u \sqrt{d_u p_u})^2}{B} }.
\]
The theorem follows.
\end{proof}

\section{Kernelisation of \qufur}
\label{sec:kernelisation}
We extend \qufur($\alpha$) to kernel regression, following an approach similar to \citet{valko2013finite}. Assume mapping $\phi: \R^d \rightarrow \Hcal$ maps the data to a reproducing kernel Hilbert space. Assume $\|\theta^*\| \le C =\tilde{O}(1)$, and $\|\phi(x)\| \le 1$, $\langle \phi(x), \theta^* \rangle^2 \le 1$, for all $x$. Define the kernel function $k(x,x')=\phi(x)^\top \phi(x'), \forall x, x' \in \R^d$. Assume the ground-truth label is generated via $y_t = \phi(x_t)^\top \theta^* + \xi_t$.

The kernelised \qufur algorithm is as follows: Let $\Qcal_t$ denote the set of indices of the queried examples up to round $t-1$. Denote $M_t = \lambda I + K_t$ where $K_t$ is the kernel matrix $[k(x, x')]_{x, x' \in \Qcal_t}$, and $\lambda=1/C^2 = \tilde{\Omega}(1)$. Define column vector $k_t = [k(x_t, x)]^\top_{x \in \Qcal_t}$. 
We predict $\hat{y}_t = \clip(k_t^\top M_t^{-1} Y_{\Qcal_t})$. Uncertainty estimate $\Delta_t=  \tilde{\eta}^2 \min\{1, \|k_t\|_{M_t^{-1}}^2)\}$, where $\|k_t\|_{M_t^{-1}}^2=\frac{1}\lambda (k(x_t, x_t)-k_t^\top M_t^{-1} k_t)$. We still query with probability $\min{\{1, \alpha \Delta_t\}}$.

A trivial regret and query guarantee is similar to Theorem 1, with $d_u$ replaced by the dimension of the support of $\phi(x)$ for $x$ in domain $u$, which is possibly infinite. Below we obtain a trade-off dependent on the \textit{effective dimension} $\tilde{d}_u$ of $\Xcal_u$ defined in equation~\eqref{eq:def_tilde_d}. For example, $\tilde{d}_u = \tilde{O}((\log{T_u})^{d_u+1})$ for the RBF kernel~\citep{srinivas2009gaussian}.

\begin{theorem}
\label{thm:kernel}
Suppose the inputs $\cbr{x_t}_{t=1}^T$ have the following structure: $[T]$ can be partitioned into $m$ disjoint nonempty subsets $\cbr{I_u}_{u=1}^m$, where for each $u$, $|I_u|=T_u$, and the effective dimension of $\cbr{x_t}_{t \in I_u}$ is $\tilde{d}_u$. If kernelised \qufur receives inputs dimension $d$, time horizon $T$, norm bound $C =\tilde{O}(1)$, noise level $\eta$, parameter $\alpha$, then, with probability $1-\delta$:\\
1. Its query complexity $Q= \tilde{O}\left(  \sum_{u=1}^m \min\{T_u, \tilde{\eta} \sqrt{\alpha \tilde{d}_u T_u}\} + 1 \right)$.\\
2. Its regret $R= \tilde{O}\del{ \sum_{u=1}^m \max\{\tilde{\eta}^2 \tilde{d}_u,\tilde{\eta} \sqrt{\tilde{d}_u T_u/\alpha}\}}$.
\end{theorem}

Let $S$ denote the set of indices for queried examples in domain $u$. Suppose $|S|=s$. If the $i$-th queried example in domain $u$ happens at time $t$, we define $a_{i,u}=\phi(x_t)$,  $\Phi_{i} = [\phi(x)^\top]^\top_{x \in \Qcal_t}$, $\Phi_{i,u} = [\phi(x)^\top]^\top_{x \in \Qcal_t \cap I_u}$, $N_{i,u} = {\Phi_{i,u}}^\top \Phi_{i,u}+\lambda I$, for all $i \in [s]$. Note that $\|k_t\|_{M_t^{-1}}^2 = a_{i,u}^\top (\Phi_i^\top \Phi_i +\lambda I) a_{i,u} \le \|a_{i,u}\|_{N_{i,u}^{-1}}^2$. We still have that
\begin{align*}
    \sum_{t \in I_u}q_t \Delta_t = \sum_{i \in S}\tilde{\eta}^2 \min\del{1, \|k_t\|_{M_t^{-1}}^2} \leq \tilde{\eta}^2 \sum_{i \in S} \min\del{1, \|a_{i,u}\|_{N_{i,u}^{-1}}^2}.
\end{align*}
We now focus on bounding $\|a_{i,u}\|_{N_{i,u}^{-1}}^2$. We use the following lemma:
\begin{lemma}[Lemma 3 of~\citet{valko2013finite}]
    For all $i \in [s]$, the eigenvalues of $N_{i, u}$ can be arranged so that $\lambda_{j, i-1} \le \lambda_{j,i}$ for all $j \ge 1$; $\lambda_{j, i} \le \lambda_{j-1,i}$ for all $j \ge 2$; $\lambda_{j,0}=\lambda$ for all $j$, and 
\begin{align*}
   \|a_{i,u}\|_{N_{i,u}^{-1}}^2 \le \left(4+\frac{6}{\lambda}\right)\sum_{j=1}^i\frac{\lambda_{j, i}-\lambda_{j, i-1}}{\lambda_{j, i-1}}.
\end{align*}
 \label{eq:valko_lem3}
\end{lemma}

Let $\Lambda_{s, j}=\sum_{i>j}({\lambda_{i, s}}-\lambda)$. The effective dimension of domain $u$ is defined as follows:
\begin{align}
    \tilde{d}_u = \min\{j: j \lambda \ln{s} >\Lambda_{s, j}\} .
    \label{eq:def_tilde_d}
\end{align}
The effective dimension is a proxy for the number of principle directions over which the projection of $x_t$'s in domain $u$ in the RKHS is spread. If they fall in a subspace of $\Hcal$ of dimension $\tilde{d}'$, then $\tilde{d}_u' \le \tilde{d}'$. More generally it captures how quickly the eigenvalues of ${\Phi_{i,u}}^\top \Phi_{i,u}$ decrease.

We prove Lemma~\ref{eq:valko_lem3} below for completeness. We use the following lemma as a black box:
\begin{lemma}[Lemma 19 of~\citet{auer2002using}]
    Let $\lambda_1 \ge \dots \ge \lambda_d \ge 0$. The eigenvalues $\nu_1, \dots, \nu_d$ of the matrix $diag(\lambda_1, \dots, \lambda_d)+z z^\top$ with $\|z\|\le 1$ can be arranged such that there are $y_{h,j} \ge 0$, $1 \le h<j \le d$, and the following holds:
    \begin{align}
        \nu_j &\ge \lambda_j \\
        \nu_j &= \lambda_j + z_j^2-\sum_{h=1}^{j-1} y_{h,j} + \sum_{k=j+1}^d y_{j,k} \label{eq:auer_eq}\\
        \sum_{h=1}^{j-1} y_{h,j} &\le z_j^2 \\
        \sum_{j=h+1}^d y_{h,j} &\le \nu_h-\lambda_h \label{eq:auer_y}\\
        \sum_{j=1}^d{\nu_j} &= \sum_{j=1}^d{\lambda_j}+\|z\|^2
    \end{align} and if $\lambda_h>\lambda_j+1$ then
    \begin{align}
        y_{h,j} &< \frac{z_j^2 z_h^2}{\lambda_h-\lambda_j-1}.
        \label{eq:auer_con}
    \end{align}
    \label{eq:auer_lem19}
\end{lemma}

\begin{proof}[Proof of Lemma~\ref{eq:valko_lem3}]
We omit the domain index subscript $u$ for clarity.
Assume $\phi=\phi_{\Ecal}$ where $\Ecal$ is some basis for $\Hcal$. Let $\Bcal$ be any basis of $\Hcal$ extended from a maximal linearly independent subset of $\{a_{j}\}_{j \le i}$. If $Q_{\Bcal\Ecal}$ denotes the change of basis matrix from $\Bcal$ to $\Ecal$ then $\Phi_{\Ecal, i}=\Phi_{\Bcal, i} Q_{\Bcal\Ecal}$ and
\begin{align*}
    \Phi_{\Ecal, i}^\top \Phi_{\Ecal,i} = Q_{\Bcal\Ecal}^\top \Phi_{\Bcal, i}^\top \Phi_{\Bcal, i}Q_{\Bcal\Ecal}
\end{align*} where $\Phi_{\Bcal, i}$, $\Phi_{\Ecal, i}$ denote $\Phi_{i}$ with respect to the basis $\Bcal$, $\Ecal$. Thus the eigenvalues of  $\Phi_{\Ecal, i}^\top \Phi_{\Ecal, i}$ do not depend on the basis, and we can focus on $\Phi_{\Bcal, i}^\top \Phi_{\Bcal, i}$, which has zeros everywhere outside its top-left $i \times i$-submatrix. Denote this submatrix as $C_{i}$. We apply Lemma~\ref{eq:auer_lem19} by setting $d=i$, $\lambda_1 \ge \dots \ge \lambda_d \ge \lambda$ as the eigenvalues of $C_{i}+\lambda I_i$, and $z$ as the first $i$ entries of the vector $ {Q_{\Bcal\Ecal}^\top}^{-1} a_i$. Our target turns into
\begin{align*}
    \|a_{i}\|_{N_{i}^{-1}}^2 = \sum_{j=1}^d{\frac{z_j^2}{\lambda_j}}
\end{align*}
For any $1 \le h<j \le d$, if $\lambda_h > \lambda_j+3$, by inequality~\eqref{eq:auer_con}, we have
\begin{align*}
    y_{h,j} \le \frac{1}{2}z_j^2 z_h^2,
\end{align*} and since $\|z\| \le 1$,
\begin{align*}
    \sum_{h:h<j, \lambda_h > \lambda_j+3} y_{h,j} \le \frac{z_j^2}{2}\sum_{h:h<j, \lambda_h > \lambda_j+3}{z_h^2} \le \frac{z_j^2}{2}.
\end{align*}
If $\lambda_h \le \lambda_j+3$, since $\lambda_j \ge \lambda$, $\lambda_j \ge \frac{\lambda}{\lambda+3} \lambda_h$, so
\begin{align*}
    \sum_{j=1}^d\sum_{h<j:\lambda_h \le \lambda_j+3} \frac{y_{h,j}}{\lambda_j} &\le \frac{\lambda+3}{\lambda} \sum_{j=1}^d \sum_{h<j:\lambda_h \le \lambda_j+3} \frac{y_{h,j}}{\lambda_h} \\ &\le \frac{\lambda+3}{\lambda} \sum_{h=1}^d \sum_{j=h+1}^d \frac{y_{h,j}}{\lambda_h} \\ &\le  \frac{\lambda+3}{\lambda} \sum_{j=1}^d \frac{\nu_j-\lambda_j}{\lambda_j}
\end{align*} where the last step is due to inequality~\eqref{eq:auer_y}.

By Equation~\eqref{eq:auer_eq},
\begin{align*}
    z_j^2 &\le \nu_j-\lambda_j+\sum_{h=1}^{j-1} y_{h,j} = \nu_j-\lambda_j+\sum_{h<j, \lambda_h > \lambda_j+3} y_{h,j} + \sum_{h<j, \lambda_h \le \lambda_j+3} y_{h,j} \\ &\le \nu_j-\lambda_j+\frac{z_j^2}{2} + \sum_{h<j, \lambda_h \le \lambda_j+3} y_{h,j},
\end{align*} so
\begin{align*}
    \sum_{j=1}^d{\frac{z_j^2}{\lambda_j}} &\le 2\sum_{j=1}^d\frac{\nu_j-\lambda_j}{\lambda_j}+ 2\sum_{j=1}^d\sum_{h<j:\lambda_h \le \lambda_j+3} \frac{y_{h,j}}{\lambda_j} \\ &\le \left(2+2 \cdot \frac{\lambda+3}{\lambda}\right) \sum_{j=1}^d\frac{\nu_j-\lambda_j}{\lambda_j} \\
    &= \left(4+\frac{6}{\lambda}\right) \sum_{j=1}^d\frac{\nu_j-\lambda_j}{\lambda_j},
\end{align*} or equivalently,
\begin{align*}
    \|a_{i}\|_{N_{i}^{-1}}^2 \le \left(4+\frac{6}{\lambda}\right)\sum_{j=1}^i\frac{\lambda_{j, i}-\lambda_{j, i-1}}{\lambda_{j, i-1}}.
\end{align*}

\end{proof}

The proof of Theorem~\ref{thm:kernel} is similar to that of Theorem~\ref{thm:upper}. We only prove the following analogue to Lemma~\ref{lem:logdet}.
\begin{lemma}
    $\sum_{i \in S} \min\del{1, \|a_{i,u}\|_{N_{i,u}^{-1}}^2} \le \tilde{O}(\tilde{d_u}).$
\end{lemma}

\begin{proof}
    By Equation~\eqref{eq:valko_lem3}, 
    \begin{align*}
        \sum_{i \in S}\|a_{i,u}\|_{N_{i,u}^{-1}}^2 &\le \left(4+\frac{6}{\lambda}\right) \sum_{i=1}^s\sum_{j=1}^i{\frac{\lambda_{j,i}-\lambda_{j,i-1}}{\lambda_{j,i-1}}} \\
        &\le \left(4+\frac{6}{\lambda}\right) \sum_{i=1}^s \left[\sum_{j=1}^{\tilde{d}_u} \frac{\lambda_{j,i}-\lambda_{j,i-1}}{\lambda_{j,i-1}} + \sum_{j=\tilde{d_u}+1}^s\frac{\lambda_{j,i}-\lambda_{j,i-1}}{\lambda_{j,i-1}}\right]
    \end{align*}
    Since we assume $C=\tilde{O}(1)$, we have $4+\frac{6}{\lambda}=\tilde{O}(1)$.
    To bound the second term, since the denominators are at least $\lambda$,
    \begin{align*}
        \sum_{i=1}^s \sum_{j=\tilde{d}_u+1}^{s}\frac{\lambda_{j,i}-\lambda_{j,i-1}}{\lambda_{j,i-1}} &\le \frac{1}{\lambda} \sum_{i=1}^s \sum_{j=\tilde{d}+1}^s(\lambda_{j,i}-\lambda_{j,i-1}) \\
        &= \frac{1}{\lambda}\sum_{j=\tilde{d}_u+1}^s(\lambda_{j, s}-\lambda) \\
        &\le \tilde{d}_u \ln{s}
    \end{align*} where the last inequality follows from Definition~\ref{eq:def_tilde_d}.
    
    To bound the first term,
    define $\alpha_{j, i}=\lambda_{j,i}-\lambda_{j,i-1}$, so the first term becomes
    \begin{align*}
        \sum_{i=1}^s \sum_{j=1}^{\tilde{d}_u} \frac{\alpha_{j, i}}{\sum_{p=1}^{i-1}{\alpha_{j, p}}+\lambda}.
    \end{align*}
    To upper bound this term, we solve the following relaxed optimization program
    \begin{align*}
        \max\left\{\sum_{i=1}^s \sum_{j=1}^{\tilde{d}_u} \frac{\alpha_{j, i}}{\sum_{p=1}^{i-1}{\epsilon_{j, p}}+\lambda}\right\} \\
        s.t. \forall i \in [s], \sum_{j=1}^{\tilde{d}_u}{\alpha_{j,i}}=\sum_{j=1}^{\tilde{d}_u}{\epsilon_{j,i}} \le 1.
    \end{align*}
    The optimal solution is $\alpha_{j,i}=\epsilon_{j,i}=1/\tilde{d}_u$, for all $j,i$. We verify this via the KKT conditions below. Write the Lagrangian 
    \begin{align*}
        L(\alpha, \epsilon, \mu, g)&=\sum_{i=1}^s \sum_{j=1}^{\tilde{d}_u} \frac{\alpha_{j, i}}{\sum_{p=1}^{i-1}{\epsilon_{j, p}}+\lambda}-\sum_{i=1}^s(\mu_i (\sum_j{\alpha_{j,i}}-\sum_j{\epsilon_{j,i}}))-\sum_{i=1}^s(g_i(\sum_j{\alpha_{j,i}}-1))\\
        \frac{\partial L}{\partial \alpha_{j,i}} &= \frac{1}{\sum_{p=1}^{i-1}{\epsilon_{j, p}}+\lambda}-\mu_i-g_i \\
        \frac{\partial L}{\partial \epsilon_{j,i}} &= -\sum_{q=i+1}^s \frac{\alpha_{j,q}  }{(\sum_{p=1}^{q-1}{\epsilon_{j, p}}+\lambda)^2}+\mu_i
    \end{align*}
    Plugging in $\alpha_{j,i}=\epsilon_{j,i}=1/\tilde{d}_u$, for all $j,i$,
    \begin{align*}
        \mu_i &= \sum_{q=i+1}^s\frac{\tilde{d}_u}{(q-1+\lambda\tilde{d}_u)^2} \ge 0 \\
        g_i &= \frac{\tilde{d}_u}{i-1+\lambda \tilde{d}_u}-\mu_i \ge 0
    \end{align*}
    
    Therefore the maximum objective value is $\tilde{d_u} \sum_{i=1}^s{\frac{1}{i-1+\lambda \tilde{d_u}}}=\tilde{O}(\tilde{d_u}\log(\frac{s}{\lambda \tilde{d_u}}+1))$. 
    
    Summing up both terms completes the proof.
\end{proof}

\section{Comparison of oracle baseline and \qufur in large budget settings}
\label{sec:large_budget}
Consider the optimization program
\begin{align}
\label{opt_program}
\min_\mu & \sum_{u=1}^m  {d_u/\mu_u},  
\text{ s.t. } \sum_{u=1}^m \mu_u T_u \leq B, \mu_u \in [0,1], \forall u \in [m].
\end{align}
\begin{theorem}
The solution to~\ref{opt_program}, $\cbr{\mu_u}_{u=1}^m$, has the following structure: there exists a constant $C$, such that  
\[ 
\mu_u = \min\del{ 1, C \sqrt{\frac{d_u}{T_u}} }.\]
\label{thm:opt}
\end{theorem}
\begin{proof}
Since the constraints are linear, define the Lagrangian $L(\mu, \lambda, \gamma)=\sum_u{\frac{d_u}{\mu_u}}+\lambda (\sum_u{T_u \mu_u}-B)+\gamma^\top (\mu-1)$, where $\lambda \in \RR$, $\gamma \in \RR^m$. By the complementary slackness condition,
\begin{enumerate}
    \item If $\gamma_u>0$, $\mu_u=1$. In this case $\gamma_u = d_u-\lambda T_u$.
    \item If $\mu<1$, $\gamma_u=0$. In this case $\mu_u=\sqrt{\frac{d_u}{\lambda T_u}}$.
\end{enumerate}
The proof is complete by taking $C=1/\sqrt{\lambda}$.
\end{proof}

Theorem~\ref{thm:opt} implies that for $B > \sum_u{\sqrt{d_u T_u} } \min{\sqrt{T_u/d_u}}$, if we always query each domain with a fixed probability, the optimal solution is to query all $T_u$ examples from domain $u$ when $\sqrt{d_u/T_u} > \tau$, and query with probability proportional to $\sqrt{d_u/T_u}$ for the rest of the domains. With this setting of $\mu$, in domain $u$, the total query complexity is $T_u \mu_u = \min\del{T_u, C \sqrt{d_u T_u}}$; the regret is 
$\tilde{\eta}^2 \frac{d_u}{\mu_u} = \max\del{\tilde{\eta}^2 d_u, \frac{\tilde{\eta}^2} C \sqrt{d_u T_u}}$.

We observe that $\qufur(\alpha)$ (query w.p. $\min\{1, \alpha \Delta_t\}$) achieves the same upper bound. Specifically, for every setting of $C$, consider $\alpha = (\frac{C}{\tilde{\eta}})^2$.
Define 
$U_1 = \cbr{u: C^2 d_u > T_u}$, and 
$U_2 = \cbr{u: C^2 d_u \leq T_u}$. 
In other words,
$U_1$ (resp. $U_2$) is the collection of domains where the domain-aware uniform sampling baseline uses query probability $\mu_u$ is $=1$ (resp. $<1$). Observe that $U_1$ and $U_2$ constitutes a partition of $[m]$.

\begin{enumerate}
\item For $u \in U_1$, the domain-aware uniform querying baseline sets $\mu_u = 1$ and has query complexity $T_u$ and regret $\tilde{\eta}^2 d_u$. 
On the other hand, QuFur($\alpha$) has the same query complexity bound of $T_u$ trivially, and has a regret of $\tilde{\eta}^2 \del{ d_u + \frac{1}{C} \sqrt{d_u T_u}} = O(\tilde{\eta}^2 d_u)$, matching the baseline performance. 

\item For $u \in U_2$, the baseline sets $\mu_u = C \sqrt{\frac{d_u}{T_u}}$, and has query complexity $C \sqrt{d_u T_u}$ and regret 
$\tilde{\eta}^2 \frac{1}{C} \sqrt{d_u T_u}$.
On the other hand, QuFur($\alpha$) has the query complexity bound of $C^2 d_u + C \sqrt{d_u T_u} = C\sqrt{d_u T_u}$, and has a regret of $\tilde{\eta}^2 \del{ d_u + \frac{1}{C} \sqrt{d_u T_u}} \leq \tilde{\eta}^2 \frac{1}{C} \sqrt{d_u T_u}$, matching the baseline performance.
\end{enumerate}

\section{Additional experimental details}
\label{sec:exp_details}

For linear classification experiments, we use the same query strategy as Algorithm~\ref{alg:1}, i.e. querying with probability $\min\{1, \alpha \Delta_t\}$. As to prediction strategy, we train a linear model with NLL loss and Adam optimizer (learning rate $0.003$, weight decay $0.001$). After each new query, we train the model on all queried data for 3 additional epochs, with batch size 64.

\begin{figure}[ht]
\begin{center}
    \begin{subfigure}{.49\textwidth}
		\centering
		\includegraphics[width=\linewidth]{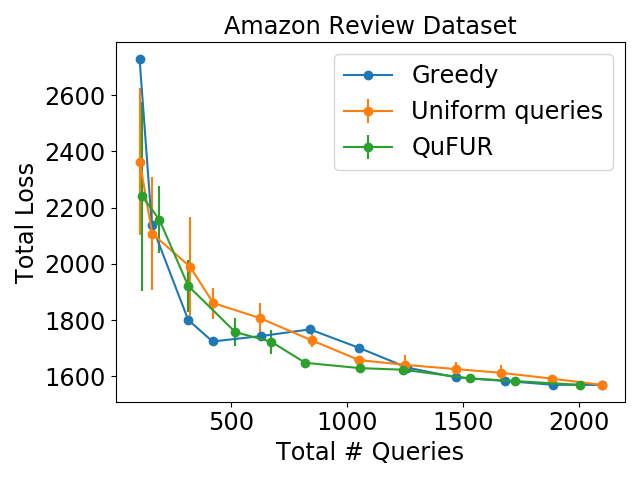}
		\caption{{\small Amazon reviews dataset with video games topic duration 1200 + grocery topic duration 600 + automobile topic duration 300. }}\label{BERT_1}
	\end{subfigure}
	\begin{subfigure}{.49\textwidth}
		\centering
		\includegraphics[width=\linewidth]{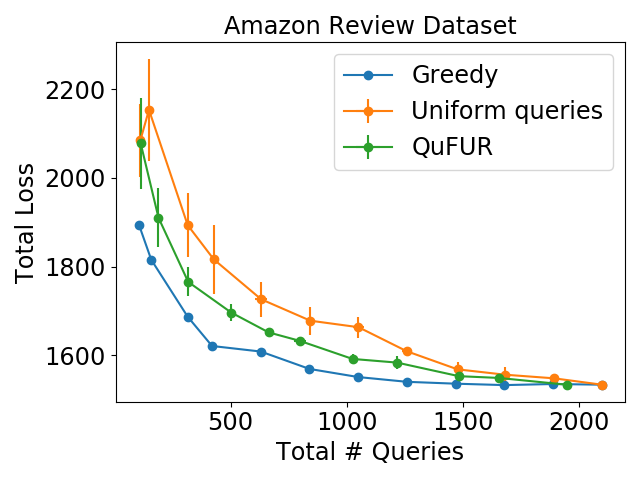}
		\caption{{\small Amazon reviews dataset with randomly shuffled inputs from all 3 topics. }}\label{BERT_2}
	\end{subfigure}
	\bigskip
	\begin{subfigure}{.49\textwidth}
		\centering
		\includegraphics[width=\linewidth]{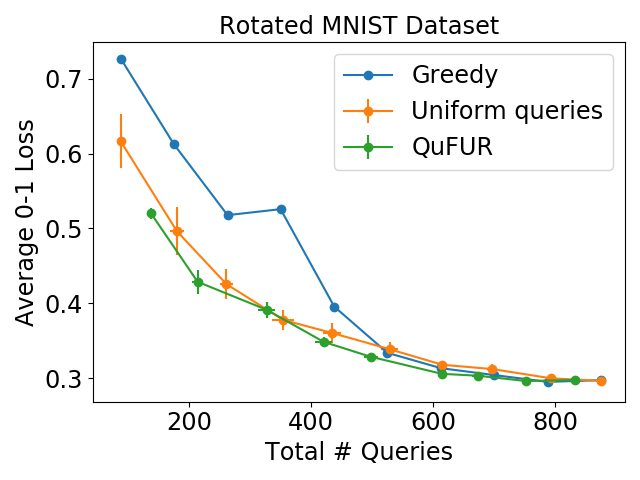}
		\caption{{\small Rotated MNIST dataset with $60^{\circ}$-rotation duration 125 + $30^{\circ}$-rotation duration 250 + no-rotation duration 500. }}\label{Rotated_1}
	\end{subfigure}
	\begin{subfigure}{.49\textwidth}
		\centering
		\includegraphics[width=\linewidth]{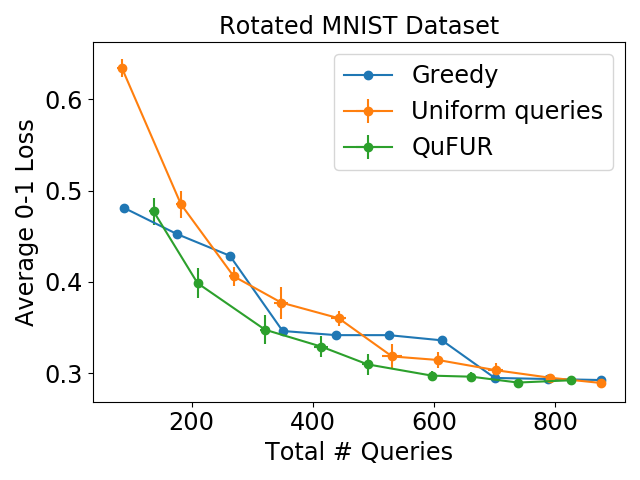}
		\caption{{\small Rotated MNIST dataset with $60^{\circ}$-rotation duration 250 + $30^{\circ}$-rotation duration 250 + $60^{\circ}$-rotation duration 125 + no-rotation duration 125 + $60^{\circ}$-rotation duration 125. }}\label{Rotated_2}
	\end{subfigure}
	\bigskip
	\begin{subfigure}{.49\textwidth}
		\centering
		\includegraphics[width=\linewidth]{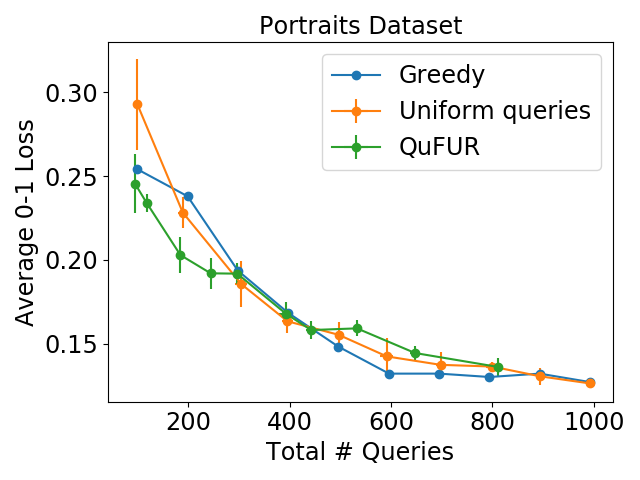}
		\caption{{\small Portraits dataset when we use the first 32, 64, 128, 256, 512 images from each time period. }}\label{Portraits_1}
	\end{subfigure}~
	\begin{subfigure}{.49\textwidth}
		\centering
		\includegraphics[width=\linewidth]{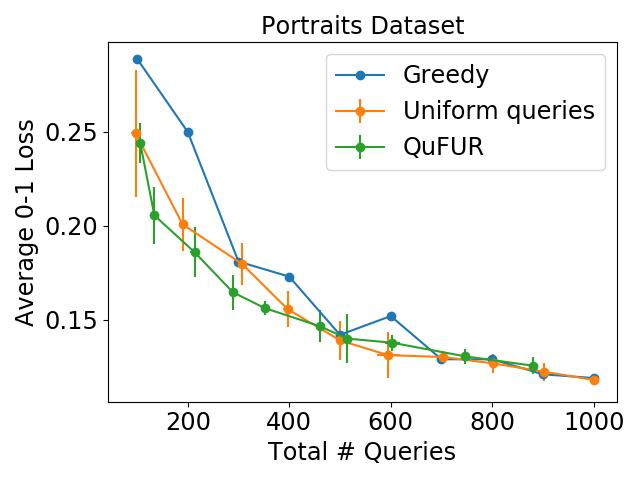}
		\caption{{\small Portraits dataset with the first 200 images from all domains. }}\label{Portraits_2}
	\end{subfigure}~
\caption{Query-loss tradeoff curves for alternative domain setups.} 
\label{fig:setups}
\end{center}
\end{figure}

Figure~\ref{fig:setups} shows the tradeoff curves for alternative domain setups on different datasets. \qufur maintains competitive performance when we reverse the order of domains (Figures~\ref{BERT_1},~\ref{Rotated_1}, and~\ref{Portraits_1}), interleave domains (Figure~\ref{Rotated_2}), and make the domains homogeneous in duration (Figure~\ref{Portraits_2}), with the exception of randomly shuffled inputs from all domains (Figure~\ref{BERT_2}). In this case, since the inputs are iid, greedy strategies can learn an accurate model early and achieve low loss. However, greedy strategies are unlikely to perform well whenever there is domain shift.

\end{document}